\documentclass[sigconf]{acmart}

\AtBeginDocument{%
  }

\setcopyright{acmcopyright}
\copyrightyear{2018}
\acmYear{2018}
\acmDOI{XXXXXXX.XXXXXXX}

\acmConference[PPDP 2022]{Principles and Practice of Declarative Programming }{June 03--05,
  2018}{Woodstock, NY}
                      
\acmPrice{15.00}
\acmISBN{978-1-4503-XXXX-X/18/06}
\usepackage{amsmath}
\usepackage{tikz}
\usetikzlibrary{trees}
\usetikzlibrary{positioning} 
\usetikzlibrary{arrows}
\usetikzlibrary{decorations.pathmorphing}
\usetikzlibrary{shapes.multipart}
\usetikzlibrary{shapes.geometric}
\usetikzlibrary{calc}
\usetikzlibrary{positioning} 
\usetikzlibrary{fit}
\usetikzlibrary{backgrounds}
\usetikzlibrary{automata}
\usepgflibrary{shapes.geometric}
\usetikzlibrary{shapes.geometric}
\usepackage{mathpartir}
\usepackage{listings}
\usepackage{proof}
\newtheorem{theorem}{Theorem}
\newtheorem{lemma}{Lemma}
\newtheorem{corollary}{Corollary}

\usepackage{todonotes}
\usepackage{graphicx}
\lstdefinelanguage{L4}
{morekeywords={
      assert 
    , class  
    , decl   
    , defn   
    , extends
    , lexicon
    , fact   
    , rule   
    , derivable
    , let   
    , in    
    , not   
    , forall
    , exists
    , if   
    , then 
    , else 
    , for  
    , true 
    , false
},    
sensitive=false,
morecomment=[l]{\#},
morestring=[b]",
}

\definecolor{mauve}{rgb}{0.88, 0.69, 1.0}

\newcommand{\etc}{\textit{etc}}
\newcommand{\ie}{\textit{i.e.\ }}

\begin{document}

\title{User Guided Abductive Proof Generation for Answer Set Programming
  Queries (Extended Version)}

\author{Avishkar Mahajan}
\orcid{0000-0002-9925-1533}
\author{Meng Weng Wong}
\orcid{0000-0003-0419-9443}
\affiliation{%
  \institution{Singapore Management University}
  \country{Singapore}
}

\author{Martin Strecker}
\orcid{0000-0001-9953-9871}
\affiliation{%
  \institution{University of Toulouse}
  \country{France}
}

\begin{abstract}
 We present a method for generating possible proofs of a query with respect to a given Answer Set Programming (ASP) rule set using an abductive process where the space of abducibles is automatically constructed just from the input rules alone. Given a (possibly empty) set of user provided facts, our method infers any additional facts that may be needed for the entailment of a query and then outputs these extra facts, without the user needing to explicitly specify the space of all abducibles. We also present a method to generate a set of directed edges corresponding to the justification graph for the query. Furthermore, through different forms of implicit term substitution, our method can take user provided facts into account and suitably modify the abductive solutions. Past work on abduction has been primarily based on goal directed methods. However these methods can result in solvers that are not truly declarative. Much less work has been done on realizing abduction in a bottom up solver like the Clingo ASP solver. We describe novel ASP programs which can be run directly in Clingo to yield the abductive solutions and directed edge sets without needing to modify the underlying solving engine. 
\end{abstract}

%%% Local Variables: 
%%% mode: latex
%%% TeX-master: "main"
%%% End: 

\begin{CCSXML}
<ccs2012>
   <concept>
       <concept_id>10003752.10003790.10003795</concept_id>
       <concept_desc>Theory of computation~Constraint and logic programming</concept_desc>
       <concept_significance>500</concept_significance>
       </concept>
   <concept>
       <concept_id>10003752.10003790.10003794</concept_id>
       <concept_desc>Theory of computation~Automated reasoning</concept_desc>
       <concept_significance>500</concept_significance>
       </concept>
   <concept>
       <concept_id>10010405.10010455.10010458</concept_id>
       <concept_desc>Applied computing~Law</concept_desc>
       <concept_significance>500</concept_significance>
       </concept>
 </ccs2012>
\end{CCSXML}

\ccsdesc[500]{Theory of computation~Constraint and logic programming}
\ccsdesc[500]{Theory of computation~Automated reasoning}
\ccsdesc[500]{Applied computing~Law}

\maketitle

%----------------------------------------------------------------------
% Main part

\section{Introduction}\label{sec:introduction}

The goal of this paper is to show how a bottom up ASP reasoner like Clingo can be used for Abductive reasoning over First Order Horn clauses. As mentioned in the abstract previous work in abductive reasoning has mostly focused on implementing abduction in a top-down manner with Prolog as the underlying engine. CIFF \cite{mancarella09:_ciff} is a prominent example of this. More recently sCASP \cite{arias19:_const_answer_set_progr_groun_applic,arias_phd_2019} has been developed as a goal directed ASP implementation that can be used for abduction but this too uses a top down method for query evaluation. However there may be use cases where one wants to know all the resulting consequences of an abductive solution to a query with respect to a rule-set. Also, as mentioned in the abstract, top-down methods can sometimes result in solvers that are not truly declarative. Therefore an abductive reasoner that uses a solver like Clingo \cite{gebser12:_answer_set_solvin_pract} can
complement the abilities of goal directed reasoners like sCASP, CIFF \etc.

This paper shows how, given an input ASP rule set, one can write a new ASP program based on that rule set which will yield abductive solutions to queries, with the input ASP rule set as the background theory. The user does not have to explicitly specify the space of abducibles. This translation from the input ASP rule set to the derived ASP program is a purely mechanical one. The key idea is to encode backward chaining over the input rules through the use of meta predicates which incorporate a notion of 'reversing' the input rules to recursively generate pre-conditions from post conditions thereby generating a maximal space of abducibles. Then having generated this maximal space of abducibles, this 'feeds into' another part of the program where we have a representation of the input rules in the normal 'forward' direction. Entailment of the specified query is then checked via an integrity constraint and a minimal set of abduced facts is returned.

The main technical challenges are dealing with situations where input rules have existential variables in pre-conditions or when the query itself has existential variables. The other challenge is to control the depth of the abducibles generation process. The work that seems to come closest to ours is \cite{schueller16:_model_variat_first_order_horn}. It too uses some similar meta predicates to encode backward chaining, and a forward representation of the rules to check for query entailment via integrity constraints.

However there are several novel features in our work.  Firstly, depth  control for abducible generation is done in a purely declarative way as part of the encoding itself without needing to call external functions or other pieces of software. Furthermore, adding facts to the program automatically gives an implicit form of term substitution where Skolem terms or
other 'place-holder' terms occurring in abducibles are replaced away so that
the resulting proof is simplified, without any need for an explicit representation of equality between terms. Past work on this topic such as \cite{schueller16:_model_variat_first_order_horn} models equality between terms via an explicit equality predicate which may become unwieldy. Another approach to dealing with existential variables encountered during the abductive proof search is to simply ground all the rules over the entire domain of constants. However, this can often lead to too many choices for what an existential variable may be substituted for which may result in unexpected/unintuitive solutions. Our method avoids both of these techniques. We present three main sets of abductive proof generation encodings. One of the encodings only supports partial term substitution whereas the other two support full term substitution. Lastly, we also present an encoding which generates a set of directed edges representing a justification 
graph  for the generated proof, where the graph can be of any desired depth.

The rest of the paper is organised as
follows. First we give a brief introduction to Answer Set Programming and Abductive reasoning then, Section~\ref{sec:abductive_proof} defines the problem being tackled more formally. Section~\ref{sec:derived_asp} presents the encodings that facilitate the abductive proof generation and directed edge generation. The sections that follow discuss some formal results regarding completeness, finiteness of abductive proof generation. We also discuss a formal result regarding term substitution. Finally Section~\ref{sec:conclusion} discusses
future work and concludes.

This is an extended version of a paper presented at PPDP~2022 \cite{ppdp_version}.

\subsection{Answer Set Programming}
Answer Set Programming (ASP) is a declarative language from the logic programming family. It is widely used and studied by knowledge representation and reasoning and symbolic AI researchers for its ability to model common sense reasoning, model combinatorial search problems etc. It incorporates the $\textit{negation-as-failure}$ operator as interpreted under the $\textit{stable model semantics}$. Clingo is a well established implementation of ASP, incorporating additional features such as $\textit{choice rules}$ and optimization statements. We shall only briefly touch upon various aspects of ASP and Clingo here. The reader may consult \cite{gebser12:_answer_set_solvin_pract} for a more thorough description. Each rule in an ASP program consists of a set of body atoms. Some of these body atoms maybe negated via the negation as failure operator $not$. Rules with no pre-conditions are called facts. Given a set of rules $R$ and a set of facts $F$, the Clingo solver computes all stable models of the ASP program $F\cup R$. For example given the fact $r(alpha)$ and the rules:
\begin{lstlisting}[frame=none]
p(X):-r(X),not q(X).
q(X):-r(X), not p(X).
\end{lstlisting}
The solver will show us 2 models or answer sets given by\\ $\{r(alpha),p(alpha)\}$ and $\{r(alpha),q(alpha)\}$. Note that as opposed to Prolog, Clingo is a bottom up solver meaning that it computes complete stable models (also known as answer sets) given any ASP program. An integrity constraint is formally speaking a rule whose post-condition is the boolean $false$. In ASP, integrity constraints are written as rules with no post-conditions and are used to eliminate some computed answer sets. For example given in the following ASP program
\begin{lstlisting}[frame=none]
r(alpha).
p(X):-r(X),not q(X).
q(X):-r(X), not p(X).
:-q(X).
\end{lstlisting}
any answer set where some instantiation of $q$ is true is eliminated. Hence we get just one answer set. $\{r(alpha),p(alpha)\}$.\\ We will now give a quick introduction to two features of Clingo that we will use throughout this paper. Namely $\textit{choice rules}$ and $\textit{weak constraints}$. Weak constraints are also often known as optimization statements. Intuitively a choice rule is a rule where if the pre-conditions are satisfied then the post-condition may or may not be made true. The post-condition of a choice rule is enclosed in curly brackets. So given the following ASP program:\begin{lstlisting}[frame=none]
r(alpha).
{q(X)}:-r(X).
\end{lstlisting}, where the rule is a choice rule the solver will give us 2 models namely $\{r(alpha)\}$, $\{r(alpha),q(alpha)\}$. If we modify the program by adding an integrity constraint like so:
\begin{lstlisting}[frame=none]
r(alpha).
{q(X)}:-r(X).
:-q(X).
\end{lstlisting}
then we get just one model $\{r(alpha)\}$.\\
Weak constraints are used in Clingo to order answer sets by preference according to the atoms that appear in them. Without going into too much detail let us just explain the meaning of one kind of weak constraint which is the only kind that we will use in the paper namely:
\begin{lstlisting}[frame=none]
:~a(X). [1@1,X]
\end{lstlisting}
Adding this to an ASP program, orders the answer sets of the program according to the number of distinct instantiations of the predicate $a$ in the answer set. The answer set with the least number of instantiations of $a$ is called the most $optimal$ answer set. 
\subsection{Abductive Reasoning}
Briefly, abduction is a reasoning process where given a background theory $T$, we wish to find a set of facts $F$ such that $F\cup T$ is consistent and $F\cup T$ entails some goal $g$ for some given entailment relation. Usually we also want $F$ to be minimal in some well defined sense. Traditional Abductive Logic Programming has a long history, but we have our own definitions of what it means to formulate and solve an abductive reasoning problem and we will make all the relevant concepts/notions precise in the sections that follow.

%%% Local Variables:
%%% mode: latex
%%% TeX-master: "main"
%%% End:

\section{Abductive Proof Generation Task}\label{sec:abductive_proof}

\subsection{Formal Setup}\label{formalsetup}

\begin{definition}[Abductive Proof Generation Task]\label{def:abductive_proof_generation_task}

Given a source ASP rule set $R$, consider the tuple $\langle R,q,U,C,N \rangle$, which we will refer to as the $\textit{Abductive Proof Generation Task}$. In this tuple,  $R$ denotes a set of input ASP rules, which we shall also refer to as the input rules or the source rules throughout the rest of this paper. $q$ is either a possibly un-ground or partially ground positive atom or, a ground negation-as-failure atom. $q$ intuitively represents the goal of our abductive reasoning process. In the context of an abductive proof generation task we may also sometimes refer to $q$ as the $query$. 
%\remms{say what intuitively $q$, $U$ etc. mean} 
The set $U$ consists of 2 subsets, $U = U_{f} \cup U_{a}$. Here $U_{f}$ is a set of user provided facts. $U_{a}$ is a set of integrity constraints that prevents certain atoms from being abduced. Throughout the rest of this paper we may sometimes just refer to the set $U$ as a whole making it clear what is contained in the subsets. $C$ denotes a set of ASP constraints which constrain which atoms may or may not appear in the complete model that results from the rules, user provided facts and abducibles. Finally we have the non-negative integer $N$. This acts as the depth control parameter for abductive proof generation. 
 \end{definition}
Given an abductive proof generation task $\langle R,q,U,C,N \rangle$,let us define what we mean by a $\textit{General Solution}$ to the task

\begin{definition}[General Solution]\label{def:abductive_proof_generation_solution} \
Given an abductive proof generation task $\langle R,q,U,C,N \rangle$, we say that $F$ is a general solution to this task if:
\begin{enumerate}
    \item If $q$ is a positive atom, it is in some answer set $\mathcal{A}$ of $F\cup U_{f}\cup R\cup C$ where the un-ground variables in $q$ have been replaced with any set of ground terms. If $q$ is a ground negation-as-failure (NAF) atom, then there exists an answer set $\mathcal{A}$ of $F\cup U_{f}\cup R\cup C$ such that $q$ is $not$ in $\mathcal{A}$.
    \item $F$ does not violate any of the integrity constraints in $U_{a}$, ie. $F$ does not contain abducibles that are specifically disallowed by the constraints in $U_{a}$.
    \item $F$ does not have any atoms whose depth level is greater than $N$.
\end{enumerate}
\end{definition} We will next state some assumptions we make on the set of input ASP rules $R$ and then also define what we mean by the depth level of an atom. After this we shall exemplify all these definitions with an example. 
\subsection{Input ASP program}
When considering any abductive proof generation task we make the following assumptions on the input ASP rule set $R$. We assume that each source ASP rule has exactly the following form:
\begin{lstlisting}[frame=none]
pre_con_1(V1),pre_con_2(V2),...,pre_con_k(Vk),
    not pre_con_k+1(Vk+1),...,not pre_con_n(Vn) -> post_con(V).
\end{lstlisting}
We further make the folllowing assumptions:

\begin{enumerate}
    \item Each pre-condition $pre\_con_{i}(V_{i})$ is atomic and so is the post-condition $post\_con(V)$.
    \item The $not$ in front of the pre-conditions denotes \textit{negation as failure} interpreted under the \textit{stable model semantics}
    \item $V_{i}$ is the set of variables occuring in the $i^{th}$ pre-condition which is either $pre\_con_{i}(V_{i})$ or $not$ $pre\_con_{i}(V_{i})$ and $V$ is the set of variables occuring in the post condition $post\_con(V)$. We assume that $V\subseteq V_{1}\cup V_{2}\cup ... \cup V_{n}$.
    \item Each variable occurring in the post condition is universally quantified over, and each variable that occurs in some pre-condition but not the post condition is existentially quantified. In particular together with (3), this means that there are no existentially quantified variables in a rule post-condition.
    \item Each variable that occurs in a negation-as-failure pre-condition also occurs in some positive pre-condition.
    \item Each input rule of the form above is assigned some unique rule id.
\end{enumerate}

We further assume that given any integrity constraint $c$ in $C$, every variable in a negation-as-failure atom in $c$ also occurs in some positive atom in $c$. We will now define the depth level of an atom with respect to a given input source ASP rule set $R$ and query $q$.

\begin{definition}[Depth level of an atom]\label{depthlvl}
Given a ASP rule set $R$ and some ground or partially ground positive atom $q$ which we shall call the query, we define a map $\phi_{R,q}$ that maps an arbitrary positive atom to a set of non-negative integers. We
will describe this map rather informally. For any atom  $q'$ such that $q'$ is obtained from $q$ by replacing the variables in $q$, with some ground terms we have, $0\in
\phi_{R,q}(q')$. Now the rest of the definition is recursive. Given an atom $a$, the
non-negative integer $n$, $n\in \phi_{R,q}(a)$ if and only if there exists a rule
$r$ in $R$ and some substitution $\theta$ of the variables in $r$
such that there exists some precondition (NAF or positive) $p$ of $r$ such 
that $\theta$ applied to $p$ gives $a$ and $\theta$ applied to the post
condition of $r$ gives some atom $a'$ where $n-1$ $\in$ $\phi_{R,q}(a')$.

Given an atom $a$ let $\phi^{min}_{R,q}(a)$ be $-1$ if the set $\phi_{R,q}(a)$ is empty
and let $\phi^{min}_{R,q}(a)$ be the minimum member of the set $\phi_{R,q}(a)$
otherwise. Then $\phi^{min}_{R,q}(a)$ is defined to be the $depth$ $level$ of $a$
with respect to the rule set $R$ and query $q$. If $q$ is a ground $NAF$ atom then given some positive atom $a$ $\phi^{min}_{R,q}(a)$ is simply given by $\phi^{min}_{R,\tilde{q}}(a)$, where $\tilde{q}$ is obtained from $q$ by removing the $not$ operator from in front of $q$.  
\end{definition}
For example if $R$ consisted of the rules
\begin{lstlisting}[frame=none]
a(X):-b(X).
b(X):-a(X).
\end{lstlisting} and $q$ was $a(X)$, meaning $q$ is un-ground then we would have that given any term $t$, $\phi_{R,q}^{min}(a(t)) = 0$, $\phi_{R,q}^{min}(b(t)) = 1$, and for any other atom $p$, $\phi_{R,q}^{min}(p) = -1$.

%%% Local Variables:
%%% mode: latex
%%% TeX-master: "main"
%%% End:

\section{Derived ASP Programs}\label{sec:derived_asp}
\subsection{A First Example}
Before we give the details of the main sets of rule translations that allow
abductive reasoning and justification generation, here is a simple example to
illustrate some key ideas and what the desired output for an abductive problem
is. Consider the rule set $R$ given by the 3 rules
\begin{lstlisting}[frame=none]
p(X,Y):-q(X,Y),s(Y).
p(X,Y):-g(X,Y).
d(X,Y):-g(X,Y).
\end{lstlisting}
Now let $q$ be $p(john,james)$, let $U$ consist only of a single constraint that disallows any
instance of the predicate $p$ from being abduced. Next let the set
$C$ contain a single constraint that disallows any instance of the predicate 
$d$, meaning that we require a stable
model of the user given facts, the abducibles, and the rules, which does not
contain any instance of the predicate $d$, finally let $N = 2$. Then for this
problem the minimal abductive solution can be represented by
\textit{abducedFact(q(john,james)), abducedFact(s(james))}, which is what we
want to get out of our encoding. Intuitively, the way we will solve this
abductive reasoning problem is by first encoding the input rules such as the
ones above in the usual forward direction. Then we will have a representation which
corresponds to 'reversing' the rules, \ie we go from post-conditions to
pre-conditions. These 'reveresed' rules
generate a maximal space of abducibles which then feed into the forward rule
translation. Finally we will have integrity constraints that ensure that the
atom which we want to be true (represented by $q$) is indeed entailed by the
abductive solution. An adapted version of this 'reversed rule' representation
also enables us to generate a set of directed edges corresponding to a
justification graph. The technical challenge in this process comes from
finding a way to deal with existential variables and the depth of abducible generation.

\subsection{Input Rule translations}

\subsubsection{Forward Translation}
Given an input ASP rule 
\begin{lstlisting}[frame=none]
pre_con_1(V1),pre_con_2(V2),...,pre_con(Vk),
   not pre_con(Vk+1),...,not pre_con_n(Vn) -> post_con(V).
\end{lstlisting}
we translate it in the following way:
\begin{lstlisting}[frame=none]
holds(post_con(V)):-holds(pre_con(V1)),...,holds(pre_con(Vk)), not holds(pre_con(Vk+1)), ..., not holds(pre_con(Vn)). 
\end{lstlisting}

We repeat this for each source ASP rule. For each constraint in $C$, we simply
enclose each atom in the constraint inside the $holds$ predicate. For example
if $C$ contains the constraint $:-b(X,Y).$, (meaning that we require an
abductive solution such that the there exists a stable model of the abduced
facts, rules and user provided facts, which contains no instantiations of the
predicate $b$), we encode that constraint as: 
\begin{lstlisting}[frame=none]
:-holds(b(X,Y)).
\end{lstlisting}

\subsubsection{Generating Abducibles}
Before diving into the details of the abducibles generation encoding let us give a brief intuition for some key meta-predicates and rules that will show up. Firstly the binary meta-predicate $query$ has as its first argument an atom which may become a candidate for abduction and as its second argument an integer corresponding roughly to the depth level of that atom with respect to $(R,q)$. Next the meta-predicate $explains$ has as its first argument, an atom which forms a pre-condition of some input rule instantiation, and as its second argument the corresponding input rule instantiation post condition. The third argument of the $explains$ meta predicate carries the depth of the atom in the first argument. The final key meta-predicate is $createSub$. $createSub$ carries information about input rule instantiations. The first argument of $createSub$ is a tuple which carries generated rule instantiations of a particular rule via the instantiations of the variables in the rule in some fixed order, the second argument of $createSub$ is again an integer depth parameter. Let us now explain the general structure of some of the rules in the abducibles generation encoding to illustrate the purpose of these meta predicates. Firstly we have abduction generation rules with the structure:
\begin{lstlisting}[frame=none]
createSub(...,N+1):-query(...,N),N<M,max_ab_lvl(M).
\end{lstlisting}
Intuitively in this rule, the first argument of the $query$ meta-predicate generates an instantiation of an input rule where that atom is the post-condition of the rule. Skolem terms or other 'place-holder' terms are used for rules with existential variables in pre-conditions. Then we have abducible generation rules with the structure:
\begin{lstlisting}[frame=none]
explains(...,N):-createSub(...,N).
\end{lstlisting}
Here a given $createSub$ atom generates an $explains$ atom where the first argument of the $explains$ atom carries an input rule pre-condition given by the rule instantiation corresponding to the $createSub$ atom and the second argument of the $explains$ atom is the instantiation of the rule post-condition. Next we have abducible generation rules with the structure:
\begin{lstlisting}[frame=none]
query(...):-explains(...).
\end{lstlisting}
Here the first two arguments of the $explains$ meta predicate, get passed on to generate two instances of the $query$ meta predicate. One can see from this that intuitively, a given $query$ atom that carries an input rule post-condition generates a $createSub$ atom that carries an input rule instantiation. This then generates an $explains$ atom whose left hand side argument is a rule pre-condition which then generates a new $query$ atom. This is the central part of the backward chaining process. The choice rule
\begin{lstlisting}[frame=none]
{abducedFact(X)}:-query(X,N).
\end{lstlisting}
Then produces the abducibles. Next we will describe the general structure of two kinds of rules which are key to enable a notion of term substitution where user-input is taken into account to simplify the generated abductive proof. First we have 
\begin{lstlisting}[frame=none]
createSub(...):-createSub(...),holds(...).
\end{lstlisting} 
and next we have 
\begin{lstlisting}[frame=none]
createSub(...):-createSub(...),query(...).
\end{lstlisting}
In the first kind of rule, arguments of $holds$ atoms can 'combine' with instances of the $createSub$ meta-predicate to yeild new instances of $createSub$. The intuition here is that if a certain instance of an input rule precondition/postcondition has  been established via the $holds$ predicate then this creates new substitutions for variables in that input rule which then leads to other input rule preconditions given by that substitution being included in the space of generated abducibles. The same intuition applies for abducible generation rules of the second kind, where instances of the $query$ predicate create new input rule substitutions. It turns out that constructing these abducible generation rules in a 'naive' way can lead to infinite answer sets, when there are skolem terms involved, even when the integer depth argument of all these meta-predicates is bounded. Hence we have different encodings for when there are skolem terms involved versus when there no skolem terms involved. Let us now get into the technical details of how these rules are constructed. First we will need a way to assign appropriate skolem terms to existential variables in pre-conditions. Given some rule in our rule set, say rule $r_{j}$, We fix some order
$O_{{j}}$ on the variables occuring in the combined set of variables from
the post and pre-conditions of the rule $r_{j}$. Now we will describe a
skolemization map that assigns an existential variable in a pre-condition of
$r_{j}$ to a skolem term. Firstly, let the rule $r_{j}$ carry unique integer
id $j$. Let $v$ be a variable that occurs in some rule precondition but not
the post condition. Then under this skolemization map, the variable $v$ gets
mapped to
\begin{lstlisting}[frame=none]
skolemFn_j_v_(V)    
\end{lstlisting}
where $V$ denotes the variables in the post-condition occuring in the order inherited from $O_{r_{j}}$. For example consider the rule $r_{1}$: 
\begin{lstlisting}[frame=none]
p(Y,X):- q(X),r(X,Y),s(Z).    
\end{lstlisting} 
Assume that this rule carries integer id $1$. Let $O_{1}$ be $[X,Y,Z]$. Then
the variable $Z$ gets mapped to the skolem term $skolemFn\_1\_Z(X,Y)$.

\subsubsection{AG1}
Given an input ASP rule $r$ in $R$
\begin{lstlisting}[frame=none]
pre_con_1(V1),pre_con_2(V2),...,pre_con(Vk),not pre_con(Vk+1),...,not pre_con_n(Vn) -> post_con(V).
\end{lstlisting}
our first set of translated abducible generation rules $AG1$ is given by the following. 
\begin{lstlisting}[frame=none]
create_subs(sub_Inst_j((V_sk),N+1):-query(post_con(V),N),max_ab_lvl(M), N<M-1.
\end{lstlisting}
Here $V_{sk}$, denotes the ordered list $O_{j}$ but with existential variables replaced by their skolem term counter parts. Here $j$ is the integer id for the rule. The integer $M=N+1$, where $N$ is the fifth entry of the tuple representing the abduction task, which represents the maximum depth of an abducible.
Next we have the following rules:
\begin{lstlisting}[frame=none]
explains(pre_con_1(V1),post_con(V),N):-create_subs(sub_Inst_t((V),N).
explains(pre_con_2(V2),post_con(V),N):-create_subs(sub_Inst_t((V),N).
...
explains(pre_con_n(Vn),post_con(V),N):-create_subs(sub_Inst_t((V),N).
\end{lstlisting}
Here $V$ denotes all the variables occuring in the rule in the order $O_{j}$.
\subsubsection{AG2}
Now we shall construct the second set
of abducible generating rules $AG2$. Given $O_{j}$ construct $F_{j}$ by adjoining the character
$V\_$ to each entry of $O_{j}$. So for our example above $F_{1}$ becomes
$[V\_X,$ $V\_Y,V\_Z]$. Given a pre-condition $p$
occurring in rule $r$ with id $j$, $M_{r,p}$ is an ordered list constructed as follows. The
$i_{th}$ element of $M_{r,p}$ is the $i_{th}$ element of $O_{j}$ if the
$i_{th}$ element of $O_{j}$ is a variable which occurs in $p$. Otherwise, the
$i_{th}$ element of $M_{r,p}$ is given by the $i_{th}$ element of $F_{r}$. Now
for each negated or positive precondition $p$ we have the following rule:
\begin{lstlisting}[frame=none]
 create_subs(sub_Inst_t(M_(r,p)),M-1):-create_subs(sub_Inst_t(F_r),N),
                       holds(p),max_ab_lvl(M), N<M.   
\end{lstlisting}
Repeat this for each pre-condition. This is the set of rules $AG2$.

%\remms{\texttt{max\_ab\_lvl} should be motivated}

\subsubsection{AG3}
Finally $AG3$ consists of just the single rule:
\begin{lstlisting}[frame=none]
query(X,N):-explains(X,Y,N),max_ab_lvl(M),N<M.
\end{lstlisting}
Let us consider another example.
Consider the input ASP rule:
\begin{lstlisting}[frame=none]
a(X):- b(X,Y,Z), not c(X), not d(Y) .    
\end{lstlisting}
Say this rule $r$ has rule id $5$. 
Let $O_{5}$ be $[X,Y,Z]$. Here the encoding for the rule $AG1$ is:
\begin{lstlisting}[frame=none]
create_subs(subs_Inst_5(X,skolemFn_5_Y(X),skolemFn_5_Z(X)),N+1):-
   query(a(X),N),max_ab_lvl(M),N<M-1.

explains(b(X,Y,Z),a(X),N):-create_subs(subs_Inst_5(X,Y,Z),N).

explains(c(X),a(X),N):-create_subs(subs_Inst_5(X,Y,Z),N).

explains(d(Y),a(X),N):-create_subs(subs_Inst_5(X,Y,Z),N).
\end{lstlisting}
AG2 is given by:

\begin{lstlisting}[frame=none]
create_subs(subs_Inst_5(X,Y,Z),M-1):-
     create_subs(subs_Inst_5(V_X,V_Y,V_Z),N), 
     holds(b(X,Y,Z)),max_ab_lvl(M),N<M.

create_subs(subs_Inst_5(X,V_Y,V_Z),M-1):-
create_subs(subs_Inst_5(V_X,V_Y,V_Z),N), holds(c(X)),max_ab_lvl(M),N<M.

create_subs(subs_Inst_5(V_X,Y,V_Z),M-1):-
create_subs(subs_Inst_5(V_X,V_Y,V_Z),N), holds(d(Y)),max_ab_lvl(M),N<M.
\end{lstlisting}

\subsubsection{Supporting code for Abduction} 
Given the original problem $\langle R,q,U,C,N\rangle$, set $M=N+1$. Then we have the following: 
\begin{lstlisting}[frame=none]
max_ab_lvl(M).
query(Q,0):-generate_proof(Q).
{abducedFact(X)}:-query(X,M).
holds(X):-abducedFact(X).

holds(X):-user_input(pos,X).
\end{lstlisting}
For any predicate $p$, say of arity $n$ such that no instance of $p$ may be abduced, we add the constraint.
\begin{lstlisting}[frame=none] 
:-abducedFact(p(X1,X2,...,Xn)).
\end{lstlisting}
If instead only a specific ground instance of $p$ or a partially ground instance of $p$ should be prevented from being abduced then we simply adapt the above constraint accordingly. For instance if $p$ is a binary predicate and we want that no instance of $p$ where the first argument is $alpha$ should be abduced we have the constraint
\begin{lstlisting}[frame=none] 
:-abducedFact(p(alpha,X)).
\end{lstlisting}
Next we add the following $weak$ constraint so that in the optimal abductive
solution as few abducibles as possible are used.
%\remms{what does $[1@1,X]$ mean?}
\begin{lstlisting}[frame=none] 
:~abducedFact(X).[1@1,X]
\end{lstlisting}

\subsubsection{Specifying the goal}
Here is the code to encode the goal of the abductive reasoning process represented by the parameter $q$. If q is a ground atom say $p(a1,a2..,an)$ for some predicate $p$ then we
have:
%\remms{\texttt{-: not goal} needs some explanation.}
\begin{lstlisting}[frame=none]
generate_proof(p(a1,a2,...,an)).
goal:-holds(p(a1,a2,..,an)).
:- not goal.
\end{lstlisting} 
Here the constraint \begin{lstlisting}[frame=none]
:- not goal. \end{lstlisting}
ensures that the abduced facts together with the input rules actually do entail the goal.
If on the other hand $q$ is un-ground or only partially ground then we have the following. Say our goal is of the form $p(a,X,b,Y,Y)$, which means that $X$,$Y$ are existential variables. Then for the example we have the following:
\begin{lstlisting}[frame=none]
generate_proof(p(a,v1,b,v2,v2)).
goal:-holds(p(a,X,b,Y,Y)).
:- not goal.
\end{lstlisting}
Here $v1$, $v2$ are fresh constants. If $q$ is a ground NAF atom say $\textit{not p}$ then we simply write 
\begin{lstlisting}[frame=none]
generate_proof(p).
goal:-not holds(p).
:- not goal.
\end{lstlisting}
Given $\langle R,q,U,C,N \rangle$ let the complete derived ASP program that uses $AG1$, $AG2$, $AG3$, the supporting code and the forward translation be called $P^{res}_{\langle R,q,U,C,N\rangle}$. We will now give a modified abduction generation encoding which can be used when no rule in $R$ contains an existential variable. As mentioned before, it turns out that using this modified encoding on rules that have existential variables can lead to infinite answer sets. After giving this modified encoding we will explain in detail the encodings with the aid of an example. 
\subsection{Extending abduction generation space for rules without existential variables}
When we have rules without existential variables, we can construct a larger
space of abducibles without worrying about our ASP programs having infinite
answer sets because there are now no skolem expressions. The encoding $AG1$ is
the same as before but now clearly there will be no skolem terms. The new
version of $AG2$ which we shall call $AG2_{exp}$ now becomes for each rule 
\begin{lstlisting}[frame=none]
create_subs(sub_Inst_t(M_(r,p)),N):-
   create_subs(sub_Inst_t(F_r),N),holds(p).   
\end{lstlisting}
Notice that as opposed to the previous encoding, here the integer argument of the $createSub$ predicate on the left hand side is $N$ as opposed to $M-1$. Repeat this for each pre-condition $p$. Then for the post-condition of the rule $p'$, we have:
\begin{lstlisting}[frame=none]
create_subs(sub_Inst_t(M_(r,p')),N):-
   create_subs(sub_Inst_t(F_r),N),holds(p').   
\end{lstlisting}
 
Here $M_{r,p'}$ is defined exactly the same way as $M_{r,p}$ for some pre-condition $p$. Next, for each rule and for each pre-condition $p$ in the rule we have.
\begin{lstlisting}[frame=none]
create_subs(sub_Inst_t(M_(r,p)),N):-
   create_subs(sub_Inst_t(F_r),N),query(p,L).   
\end{lstlisting}

For the post-condition $p'$ we have:
\begin{lstlisting}[frame=none]
create_subs(sub_Inst_t(M_(r,p')),N):-
   create_subs(sub_Inst_t(F_r),N),query(p',L).   
\end{lstlisting}

This completes the encoding $AG2_{exp}$. The adapted version of $AG3$,
$AG3_{exp}$, is given by adding to $AG3$ one extra rule. So $AG3_{exp}$ is: 
\begin{lstlisting}[frame=none]
query(X,N):-explains(X,Y,N),max_ab_lvl(M),N<M.
query(Y,N-1):-explains(X,Y,N),max_ab_lvl(M),0<N,N<M.
\end{lstlisting}
Given $\langle R,q,U,C,N \rangle$ let the complete ASP program that uses $AG1_{exp}$, $AG2_{exp}, AG3_{exp}$, the supporting code and the forward translation be called $P_{<R,q,U,C,N>}^{exp}$
\subsection{Discussion of Abduction space generation}

\subsubsection{Full term substitution}
We first give an example of the expanded abduction space encoding to explain
the intuition behind various parts of the encoding. Consider the rule set
below that has no existential variables but which has negation as failure and
where the goal is un-ground.
\begin{lstlisting}[frame=none]
relA(X,Y):-relB(X,Y), relD(Y), not relE(Y).
relE(Y):-relD(Y), not relF(Y).
\end{lstlisting}
Let the goal $q$ be $relA(P,R)$, where $P$,$R$ are un-ground existential
variables. Next suppose that the only constraint on abducibles is that no
instantiation of $relA$ can be abduced and further suppose that the set of
user provided facts is initially empty. Finally let $N=4$. Here is the
complete encoding for this problem.

\begin{lstlisting}[numbers=left]
max_ab_lvl(5).
% Encoding the goal
generate_proof(relA(v1,v2)).
query(X,0):-generate_proof(X).
goal:-holds(relA(P,R)).
:- not goal.

% forward translation
holds(relA(X,Y)) :- holds(relB(X, Y)),holds(relD(Y)), not holds(relE(Y)).
holds(relE(Y)) :- holds(relD(Y)), not holds(relF(Y)).

% AG1_exp
createSub(subInst_r1(X,Y),N+1) :- query(relA(X,Y),N),max_ab_lvl(M),N<M-1. (* \label{abdInstR1} *)
createSub(subInst_r2(Y),N+1) :- query(relE(Y) ,N),max_ab_lvl(M),N<M-1. (* \label{abdInstR2} *)

explains(relB(X, Y), relA(X,Y) ,N) :- createSub(subInst_r1(X,Y),N). (* \label{abdExplR1start} *)
explains(relD(Y), relA(X,Y) ,N) :- createSub(subInst_r1(X,Y),N).
explains(relE(Y), relA(X,Y) ,N) :- createSub(subInst_r1(X,Y),N). (* \label{abdExplR1end} *)


explains(relD(Y), relE(Y) ,N) :- createSub(subInst_r2(Y),N). (* \label{abdExplR2start} *)
explains(relF(Y), relE(Y) ,N) :- createSub(subInst_r2(Y),N). (* \label{abdExplR2end} *)


% AG2_exp for rule 1

createSub(subInst_r1(X,Y),N) :- createSub(subInst_r1(V_X,V_Y),N), holds(relA(X,Y)).
createSub(subInst_r1(X,Y),N) :- createSub(subInst_r1(V_X,V_Y),N), holds(relB(X,Y)).
createSub(subInst_r1(V_X,Y),N) :- createSub(subInst_r1(V_X,V_Y),N), holds(relD(Y)).
createSub(subInst_r1(V_X,Y),N) :- createSub(subInst_r1(V_X,V_Y),N), holds(relE(Y)).

createSub(subInst_r1(X,Y),N) :- createSub(subInst_r1(V_X,V_Y),N), query(relA(X,Y),L).
createSub(subInst_r1(X,Y),N) :- createSub(subInst_r1(V_X,V_Y),N), query(relB(X,Y),L).
createSub(subInst_r1(V_X,Y),N) :- createSub(subInst_r1(V_X,V_Y),N), query(relD(Y),L).
createSub(subInst_r1(V_X,Y),N) :- createSub(subInst_r1(V_X,V_Y),N), query(relE(Y),L).

% AG2_exp for rule 2

createSub(subInst_r2(Y),N) :- createSub(subInst_r2(V_Y),N), holds(relE(Y)).
createSub(subInst_r2(Y),N) :- createSub(subInst_r2(V_Y),N), holds(relD(Y)).
createSub(subInst_r2(Y),N) :- createSub(subInst_r2(V_Y),N), holds(relF(Y)).


createSub(subInst_r2(Y),N) :- createSub(subInst_r2(V_Y),N), query(relE(Y),L).
createSub(subInst_r2(Y),N) :- createSub(subInst_r2(V_Y),N), query(relD(Y),L).
createSub(subInst_r2(Y),N) :- createSub(subInst_r2(V_Y),N), query(relF(Y),L).

% AG3_exp
query(X,N):-explains(X,Y,N),max_ab_lvl(M),N<M.(* \label{abdQuery} *)
query(Y,N-1):-explains(X,Y,N),max_ab_lvl(M),N<M,0<N.(* \label{abdAbdHold} *)

% Supporting code
{abducedFact(X)}:-query(X,N).  (* \label{abdChoice} *)
holds(X):-abducedFact(X).      (* \label{abdHoldsAbduced} *)
holds(X):-user_input(pos,X).   (* \label{abdHoldsUI} *)


:~abducedFact(Y).[1@1,Y]    (* \label{abdAbducedFact} *)
:-abducedFact(relA(X,Y)).

\end{lstlisting}

We will now discuss various parts of the encoding.

As mentioned earlier, the general idea is to recursively generate a
maximal space of abducibles by 'reversing' the rules and then checking via the
Forward Translation and encoding of the goal, which abducibles are needed
for entailment of the original query.  
More specifically in line with the intuitive discussion from before, any atom
of the form $query(h,i)$ generates an input rule instantiation where $h$ is
the post-condition that particular rule instantiation. Such rule
instantiations are represented by the $createSub$ atom. In the example above,
this is done via lines \ref{abdInstR1}, \ref{abdInstR2} of the encoding. Line \ref{abdInstR1} corresponds to
instantiations of rule 1 and line \ref{abdInstR2} corresponds to instantiations of rule
2. Then any such $createSub$ atom, generates the appropriate set of $explains$
atoms. This is lines \ref{abdExplR1start}-\ref{abdExplR1end} for rule 1 in the example, and lines \ref{abdExplR2start}, \ref{abdExplR2end} for
rule 2. The first argument of an $explains$ atom is a pre-condition or body
atom corresponding to the rule instantiation given by the $createSub$
atom. The second argument is the post-condition or head of the rule
instantiation. We have one such $explains$ atom for each rule
pre-condition. Via line \ref{abdQuery}, the first argument of an $explains$ atom becomes
the first argument of a $query$ atom. This new $query$ atom then recursively
generates more $query$ atoms via the process described. Any $query$ atom
corresponds to a candidate for abduction via the choice rule in line \ref{abdChoice}. Any
fact which is abduced must hold due to line \ref{abdHoldsAbduced}. At this point, before moving
ahead let us first briefly comment further upon the integer arguments occuring in the
$explains$, $createSub$ and $query$ atoms.

The integer parameter roughly
represents the depth of an abducible in the proof graph of the original
query. When a $query$ atom carrying the post-condition of a rule generates a
rule instantiation like in line \ref{abdInstR1} for example, the integer argument of the
corresponding $createSub$ atom increases by one. Then an $explains$ atom
derived from the application of a rule like line \ref{abdExplR1start} retains the same integer
argument and so does the corresponding fresh $query$ atom generated from the
application of the rule on line \ref{abdQuery}. Note that a fresh $query$ atom can only be
created from an $explains$ atom if the integer parameter of the $explains$
atom is less than $M$. The use of these integer parameters is important
when we need skolem functions/terms in our abducible generation encoding due
to having rules with existentially quantified variables in pre-conditions. The
use of these integer parameters allows us to control the depth of the
abducible generating space thus preventing infinite answer sets even in the
presence of skolem functions. We will discuss this more later on. For now let
us turn our attention to some of the other parts of the encoding. The
$AG2_{exp}$ encoding enables a notion of implicit term substitution
in (minimal) abductive solutions. This set of rules creates new instantiations
of the input rules based on which other atoms are true. As stated earlier, creating new instantiations of the core input rules via the $createSub$ atoms, then allows
new abducibles to be added to the generated space of abducibles. Let us
illustrate some of these ideas with an example. Upon running the above ASP
program as the optimal solution given by the solver is: 

\begin{lstlisting}[frame=none]
abducedFact(relD(v2)) 
abducedFact(relB(v1,v2)) 
abducedFact(relF(v2))    
\end{lstlisting}
Now if $relB(john,james)$ is added to the set of user provided facts then firstly, due to line \ref{abdHoldsUI} $holds(relB(john,james))$ becomes true. Then we have the following instantiation of line 28.
\begin{lstlisting}[frame=none]
createSub(subInst_r1(john,james),1):-
   createSub(subInst_r1(v1,v2),1),holds(relB(john,james)). 
\end{lstlisting}
Hence due to lines \ref{abdExplR1start} and \ref{abdAbdHold}, the atom
$query(relA(john,james),0)$ becomes true. This leads to the atoms
$query(relD(james),1)$ and $query(relF(james),2)$ becoming true. Hence the atoms $relD(james)$ and $relF(james)$ become part of the space of abducibles and the solver gives us the new optimal solution:
\begin{lstlisting}[frame=none]
abducedFact(relD(james)) 
abducedFact(relF(james)) 
\end{lstlisting}
On the other hand if we instead add the fact $relF(mary)$ to the initially empty set of user provided facts then we get the follwing instantiation of line 41:
\begin{lstlisting}[frame=none]
createSub(subInst_r2(v2),2):-createSub(subInst_r2(v2),2),
holds(relF(mary)).
\end{lstlisting}
Hence the atom $createSub(subInst\_r2(mary),2)$ becomes true. Via line 22, and line \ref{abdAbdHold} the atom $query(relE(mary),1)$ becomes true. We thus get the following instantiation of line 35:
\begin{lstlisting}[frame=none]
createSub(subInst_r1(v1,mary),1):- createSub(subInst_r1(v1,v2),1), query(relE(mary),1).   
\end{lstlisting}
Thus the atom $createSub(subInst\_r1(v1,mary),1)$ becomes true, which then via say line \ref{abdExplR1start} and line \ref{abdAbdHold} causes the atom\\ 
$query(relA(v1,mary),0)$ to become true. Now, because of\\ $query(relA(v1,mary),0)$, 
$relD(mary)$, $relB(v1,mary)$ become part of the space of abducibles and the solver gives us the optimal abductive solution: 
\begin{lstlisting}[frame=none]
abducedFact(relD(mary)) 
abducedFact(relB(v1,mary))
\end{lstlisting}
and a similar result is obtained if we add an instance of the predicate $relD$
to the initially empty set of user provided facts. Thus with this encoding we
have full implicit term substitution. The place holder or 'dummy' variables
$v1$, $v2$, always get replaced away in the optimal abductive solution based
on the user provided facts.
%\remms{But one can still get an abduced fact containing a Skolem function?} 
A subtle point here is that there is no notion of
equality between terms. We are not setting $v2 = mary$. We are instead
enlarging the space of abducibles in a systematic way based on user provided
facts so that a more optimal solution which involves replacing the term $v2$
for the term $mary$ can be realized. Note that adding an 'unrelated' fact such
as say $relG(mary)$ will not enlarge the space of abducibles in any way. So in
some sense what we have is a method to enlarge the space of abducibles in an
'economical' way while still supporting a notion of term substitution. We will
formulate and prove a formal result regarding this notion of term substitution
later on.

\subsubsection{Partial term substitution}

When skolem terms/function are used to handle existential variables, we have
to use the non- expanded abducible generation encoding which forces us to give
up on complete term substituion.  This is because having the complete term
substitution mechanism can result in programs that have infinitely large
abducible spaces. To recover finiteness of the space of the abducibles we have
to forgo full term substitution. What we get instead is a kind of partial term
substitution mechanism where skolem terms may only sometimes be substituted
for user provided terms.  First let us examine why in the presence of skolem
functions, even a subset of the expanded abduction generation encoding can
lead to infinite answer sets.

Consider the rule set consisting of just the single rule 
\begin{lstlisting}[frame=none]
relA(X):-relB(X,Y),relA(Y).
\end{lstlisting}

Suppose the goal is $a(john)$
Consider the encoding below, which is a subset of the expanded abducible generation encoding.
\begin{lstlisting}[numbers=left]
max_ab_lvl(5).
query(relA(bob),0).
:-not holds(relA(bob)).

holds(relA(X)) :- holds(relB(X, Y)),holds(relA(Y)).

explains(relB(X, Y), relA(X) ,N) :- createSub(subInst_r1(X,Y),N).
explains(relA(Y), relA(X) ,N) :- createSub(subInst_r1(X,Y),N).


createSub(subInst_r1(X,skolemFn_r1_Y(X)),N+1) :- query(relA(X) ,N),max_ab_lvl(M),N<M-1.

createSub(subInst_r1(X,Y),N) :- createSub(subInst_r1(V_X,V_Y),N), holds(relB(X, Y)).
createSub(subInst_r1(V_X,Y),N) :- createSub(subInst_r1(V_X,V_Y),N),holds(relA(Y)).


query(X,N):-explains(X,Y,N),max_ab_lvl(M),N<M.
{abducedFact(X)}:-query(X,N).
holds(X):-abducedFact(X).
holds(X):-user_input(pos,X).

:~abducedFact(Y).[1@1,Y]
:-abducedFact(relA(bob)).
\end{lstlisting}

Due to line 11 in the encoding we get the atom\begin{lstlisting}[frame=none]
createSub(subInst_r1(bob,skolemFn_r1_Y(bob)),1).
\end{lstlisting}
Then due to lines 7 and 8 of the encoding we get the atoms \begin{lstlisting}[frame=none]
query(relA(skolemFn_r1_Y(bob)),1) query(relB(bob,skolemFn_r1_Y(bob)),1).
\end{lstlisting} 

Then via lines 11 and 7 and due to the atom
\begin{lstlisting}[frame=none]
query(relA(skolemFn_r1_Y(bob)),1)
\end{lstlisting} we get the atom \begin{lstlisting}[frame=none]
query(relB(skolemFn_r1_Y(bob), skolemFn_r1_Y(skolemFn_r1_Y(bob))),2)
\end{lstlisting} 
Then due to lines 18, 19, we get the atom \begin{lstlisting}[frame=none]
holds(relB(skolemFn_r1_Y(bob), skolemFn_r1_Y(skolemFn_r1_Y(bob))
\end{lstlisting}
Then due to line 13 of the encoding and the atom \begin{lstlisting}[frame=none]
createSub(subInst_r1(bob,skolemFn_r1_Y(bob)),1)
\end{lstlisting} we get the atom
\begin{lstlisting}[frame=none]
createSub(subInst_r1(skolemFn_r1_Y(bob), skolemFn_r1_Y(skolemFn_r1_Y(bob))),1)
\end{lstlisting}
Then due line 8 and line 17, we get the atom \begin{lstlisting}[frame=none]
query(relA(skolemFn_r1_Y(skolemFn_r1_Y(bob)),1)
\end{lstlisting}

In this way we can see that with the encoding above we would have answer sets
that contain atoms of the form\\ $query(relA(skolemFn\_r1\_Y(...),1)$ for
aribtrarily large skolem function nesting depth. Hence the encoding above
leads to infinitely large answer sets.

Intuitively, the core problem is lines like 13, 14 where the skolem depth of
terms in the $createSub$ predicate has no relation with the integer argument
of the $createSub$ predicate, thus allowing for abducibles, where the skolem
depth of the arguments inside predicates can be arbitrarily large despite
having a finite maximum abduction depth level. The solution to this problem
then is to replace the $N$ occuring as the integer argument of the $createSub$
predicate in the head of the rule on lines 13, 14 with $M-1$, where the $M$
corresponds to the argument of $max\_ab\_lvl$. This means that $query$ atoms
which occur due to the use of rules like line 13, 14 cannot further cause
fresh $query$ atoms to be added to the abducibles space via rules like the one
on line 11.

As a result of this however we lose complete term substitution. Consider the following abduction problem. $R$ is given by the following ASP rules:
\begin{lstlisting}[frame=none]
relA(P):-relB(P,R),relD(R).
relB(P,R):-relA(R),relC(P).
\end{lstlisting}
Let $q$ be the atom $relA(john)$, let $U$ consist of the constraints $:-abducedFact(relA(john)).$, $:-abducedFact(relB(X,Y)).$ meaning that $q$ cannot itself be abduced and no instantiation of the predicate $relB$ can be abduced. Let the set of user provided facts be empty for now. Let the set $C$ also be empty and let $N = 4$. This is the non expanded abduction encoding for this problem. 
\begin{lstlisting}[numbers=left]
max_ab_lvl(5).

% Encoding the goal
generate_proof(relA(john)).
goal:-holds(relA(john)).
:-not goal.
query(X,0):-generate_proof(X).

% Core rule translation
holds(relA(P)) :- holds(relB(P, R)),holds(relD(R)).
holds(relB(P, R)) :- holds(relA(R)), holds(relC(P)).

% AG1
createSub(subInst_r1(P,skolemFn_r1_R(P)),N+1) :- query(relA(P) ,N),max_ab_lvl(M),N<M-1.
createSub(subInst_r2(P,Q),N+1) :- query(relB(P, Q) ,N),max_ab_lvl(M),N<M-1.



explains(relB(P, R), relA(P) ,N) :- createSub(subInst_r1(P,R),N).
explains(relD(R), relA(P) ,N) :- createSub(subInst_r1(P,R),N).


explains(relA(R), relB(P,R) ,N) :- createSub(subInst_r2(P,R),N).
explains(relC(P), relB(P,R) ,N) :- createSub(subInst_r2(P,R),N).


% AG2 for rule 1
createSub(subInst_r1(P,R),M-1) :- createSub(subInst_r1(V_P,V_R),N), N<M, holds(relB(P, R)),max_ab_lvl(M).
createSub(subInst_r1(V_P,R),M-1) :- createSub(subInst_r1(V_P,V_R),N), N<M, holds(relD(R)),max_ab_lvl(M).

% AG2 for rule 2
createSub(subInst_r2(V_P,R),M-1) :- createSub(subInst_r2(V_P,V_R),N), N<M, holds(relA(R)),max_ab_lvl(M).
createSub(subInst_r2(P,V_R),M-1) :- createSub(subInst_r2(V_P,V_R),N), N<M, holds(relC(P)),max_ab_lvl(M).

% AG3
query(X,N):-explains(X,Y,N),max_ab_lvl(M),N<M.

% Supporting code
{abducedFact(X)}:-query(X,N).
holds(X):-abducedFact(X).
holds(X):-user_input(pos,X).


:~abducedFact(Y).[1@1,Y]
:-abducedFact(relA(john)).
:-abducedFact(relB(X,Y)).


\end{lstlisting}
Running this program in Clingo, we get the output 
\begin{lstlisting}[frame=none]
abducedFact(relC(john))
abducedFact(relD(skolemFn_r1_R(skolemFn_r1_R(john))))
abducedFact(relA(skolemFn_r1_R(skolemFn_r1_R(john))))    
\end{lstlisting}
as the solution with the least number of abducibles.
Now adding $relC(john)$ as a user provided fact gives the following smaller abductive solution
\begin{lstlisting}[frame=none]
abducedFact(relD(skolemFn_r1_R(skolemFn_r1_R(john))))
abducedFact(relA(skolemFn_r1_R(skolemFn_r1_R(john))))    
\end{lstlisting}
Now if we further add $relA(mary)$ to the set of user provided facts then we
get as a minimal abductive solution the answer $relD(mary)$. This is because
by after adding these facts, $holds(relB(john,mary))$ becomes true. Then by
line 28 of the encoding\\
$createSub(subInst\_r1(john,mary),4)$ becomes
true. Then by line 20, and line 36 $query(relD(mary),4)$ becomes true which
then gives us the minimal abductive solution. However if instead of adding
the fact $relA(mary)$ we instead add the fact $relD(mary)$, then we do not get
a  substitution of terms and the minimal abductive
solution is still
\begin{lstlisting}[frame=none]
abducedFact(relD(skolemFn_r1_R(skolemFn_r1_R(john))))
abducedFact(relA(skolemFn_r1_R(skolemFn_r1_R(john))))    
\end{lstlisting} 
This is because by line 29, the atom\\ $createSub(subInst\_r1(john,mary),4)$ becomes true,
which due to line 19 and line 36 makes $query(relB(john,mary),4)$
true. However now this cannot cause the atom $query(relA(mary),5)$ to become
true because line 15 cannot apply due to the constraint on the integer
argument of the $query$ atom. So what we have can be regarded as a partial
term substitution mechanism.

\subsubsection{Replacing skolem functions by a single constant}

Let us see how having term substitution as a derived effect via enlargement of
the space of abducibles rather than doing term substitution through an
explicit equality predicate allows us to better handle problems where the core
rules have existential variables but we do not wish to use skolem functions in
the abductive reasoning process. Recall that not having skolem functions
allows us to get full term substitution without the possiblity of infinitely
large answer sets. Consider the problem $,\langle R,q,U,C,N \rangle$ where $R$ is the
following input rule set: 

\begin{lstlisting}[frame=none]
relA(X):-relB(X,Y),relC(X,Y).
relB(X,Y):-relD(X,Y,Z),relE(X,Y,Z).
\end{lstlisting}

let our $q$ be $relA(john)$. Let the initial set of user provided facts be empty, furthermore, suppose that no instance of $relA$ or $relB$ may be abduced. Finally let the set $C$ be empty and let $N=4$.  
Consider the following encoding
\begin{lstlisting}[numbers=left]
max_ab_lvl(5).
% Encoding the goal
generate_proof(relA(john)).
query(X,0):-generate_proof(X).
goal:-holds(relA(john)).
:- not goal.



% Core rule translation
holds(relA(X)) :- holds(relB(X, Y)),holds(relC(X,Y)).
holds(relB(X,Y)) :- holds(relD(X,Y,Z)), holds(relE(X,Y,Z)).

% AG1_exp
createSub(subInst_r1(X,extVar),N+1) :- query(relA(X) ,N),max_ab_lvl(M),N<M-1.
createSub(subInst_r2(X,Y,extVar),N+1) :- query(relB(X,Y) ,N),max_ab_lvl(M),N<M-1.

explains(relB(X, Y), relA(X) ,N) :- createSub(subInst_r1(X,Y),N).
explains(relC(X, Y), relA(X) ,N) :- createSub(subInst_r1(X,Y),N).

explains(relD(X,Y,Z), relB(X,Y) ,N) :- createSub(subInst_r2(X,Y,Z),N).
explains(relE(X,Y,Z), relB(X,Y) ,N) :- createSub(subInst_r2(X,Y,Z),N).


% AG2_exp for rule 1

createSub(subInst_r1(X,V_Y),N) :- createSub(subInst_r1(V_X,V_Y),N), holds(relA(X)).
createSub(subInst_r1(X,Y),N) :- createSub(subInst_r1(V_X,V_Y),N), holds(relB(X,Y)).
createSub(subInst_r1(X,Y),N) :- createSub(subInst_r1(V_X,V_Y),N), holds(relC(X,Y)).

createSub(subInst_r1(X,V_Y),N) :- createSub(subInst_r1(V_X,V_Y),N), query(relA(X),L).
createSub(subInst_r1(X,Y),N) :- createSub(subInst_r1(V_X,V_Y),N), query(relB(X,Y),L).
createSub(subInst_r1(X,Y),N) :- createSub(subInst_r1(V_X,V_Y),N), query(relC(X,Y),L).


% AG2_exp for rule 2

createSub(subInst_r2(X,Y,V_Z),N) :- createSub(subInst_r2(V_X,V_Y,V_Z),N), holds(relB(X,Y)).
createSub(subInst_r2(X,Y,Z),N) :- createSub(subInst_r2(V_X,V_Y,V_Z),N), holds(relD(X,Y,Z)).
createSub(subInst_r2(X,Y,Z),N) :- createSub(subInst_r2(V_X,V_Y,V_Z),N), holds(relE(X,Y,Z)).


createSub(subInst_r2(X,Y,V_Z),N) :- createSub(subInst_r2(V_X,V_Y,V_Z),N), query(relB(X,Y),L).
createSub(subInst_r2(X,Y,Z),N) :- createSub(subInst_r2(V_X,V_Y,V_Z),N), query(relD(X,Y,Z),L).
createSub(subInst_r2(X,Y,Z),N) :- createSub(subInst_r2(V_X,V_Y,V_Z),N), query(relE(X,Y,Z),L).

% AG3_exp
query(X,N):-explains(X,Y,N),max_ab_lvl(M),N<M.
query(Y,N-1):-explains(X,Y,N),max_ab_lvl(M),N<M,0<N.

% Supporting code
{abducedFact(X)}:-query(X,N).
holds(X):-abducedFact(X).
holds(X):-user_input(pos,X).


:~abducedFact(Y).[1@1,Y]
:-abducedFact(relA(X)).
:-abducedFact(relB(X,Y)).

\end{lstlisting}

Note that in lines 15, 16 instead of using skolem functions we use a single
fresh constant $extVar$ to represent the existential variable in both
rules. Now, when we run the program we get the following optimal solution

\begin{lstlisting}[frame=none]
abducedFact(relC(john,extVar))
abducedFact(relE(john,extVar,extVar))
abducedFact(relD(john,extVar,extVar))    
\end{lstlisting}

Now because, term substitution is only a derived effect and there is no
equality relation, it is possible for different instances of $extVar$ to get
replaced (or not) by different constants upon the addition of some user
provided facts. For instance upon adding the fact $relD(john,james,mary)$, we get
the optimal solution:

\begin{lstlisting}[frame=none]
abducedFact(relC(john,james))
abducedFact(relE(john,james,mary))    
\end{lstlisting}

So some instances of $extVar$ from the original solution have been replaced by 'james' and others by 'mary'. What this means is that each occurence of $extVar$ in the original solution can be thought of as simply a place-holder for a term where each instance maybe a placeholder for a different term. When we use skolem functions instead this is simply more explicit because we have different skolem terms representing different existential variables. More formally in the first solution the variables $[X,Y,Z]$ get mapped to $[john, extVar,extVar]$ respectively. Upon the addition of the extra fact the we get the mapping $[X,Y,Z]\rightarrow[john,james,mary]$. Using an equality relation to get from the first solution to the second would be impossible because we would need both the following equalities to hold: $extVar = james$, $extVar = mary$. (Of course the above solution could be obtained if one simply grounds the rules over the entire domain of constants but as mentioned in the introduction, the methods in this paper are aimed at avoiding such a naïve grounding as in general, one may get too many substitutions for existential variables)

Given $<R,q,U,C,N>$, let this ASP program where we use $AG1_{exp}$ , $AG2_{exp}$, $AG3_{exp}$ but replace all use of skolem terms with $extVar$ be called $P_{<R,q,U,C,N>}^{semi-res}$.\\ We will now turn to the problem of generating a set of directed edges corresponding the computed abductive solution.

\subsection{Generating Justification Trees}

Given a source rule 
\begin{lstlisting}[frame=none]
pre_con_1(V1),pre_con_2(V2),...,pre_con_k(Vk),not pre_con_k+1(Vk+1),...,not pre_con_n(Vn) -> 
post_con(V).
\end{lstlisting}

For each positive pre-condition $pre\_con\_u(V_{u})$, we add the following ASP rule:

\begin{lstlisting}[frame=none]
causedBy(pos,pre_con_u(Vu), post_con(V),N+1):-holds(post_con(V)), holds(pre_con_1(V1)),
holds(pre_con_2(V2)),...,holds(pre_con_k(Vk)),not holds(pre_con_k+1(Vk+1)),...,
not holds(pre_con_n(Vn)),justify(post_con(V),N).   
\end{lstlisting}
For each negative precondition $pre\_con\_f(V_{f})$ we add the following ASP rule: 
\begin{lstlisting}[frame=none]
causedBy(neg,pre_con_f(Vf), post_con(V),N+1):-holds(post_con(V)), holds(pre_con_1(V1)),
holds(pre_con_2(V2)),...,holds(pre_con_k(Vk)),not holds(pre_con_k+1(Vk+1)),...,
not holds(pre_con_n(Vn)), justify(post_con(V),N).
\end{lstlisting}

\subsubsection{Supporting code for justification tree}
\mbox{}

\begin{lstlisting}[frame=none]
justify(X,N):-causedBy(pos,X,Y,N), not user_input(pos,X),N<M, max_graph_lvl(M).
directedEdge(Sgn,X,Y):-causedBy(Sgn,X,Y,M).

justify(X,0):-gen_graph(X),not user_input(pos,X).

directedEdge(pos,userFact,X):-directedEdge(pos,X,Y), user_input(pos,X).

directedEdge(pos,userFact,X):-gen_graph(X), user_input(pos,X).
\end{lstlisting}

\subsection{Discussion of Justification generation}

The intuition for the justification graph encoding is that given some user
provided facts $F$ and an input rule set $R$, an atom $a$ is only contained in
a stable model $M$ of $F$, $R$ if either $a$ is in $F$ or there exists some
rule $r$ in $R$ such that for some ground instantiation $r_{g}$ of $r$, all
the pre-conditions of $r_{g}$, (ie. the body atoms) are true in $M$ and the
post-condition (ie. head) of $r_{g}$ is $a$. Here the truth value of $NAF$
atoms is interpreted in the usual way. The edges for the justification graph
are calculated recursively. An atom $justify(h,k)$ represents the fact that
$holds(h)$ needs to be justified. If $r_{g}$ is a ground instantiantion of an
input rule where the post condition of $r_{g}$ is $h$ and all the
pre-conditions of $r_{g}$ are true then for every positive precondition
$b_{i}$ of $r_{g}$ we have the atom $causedBy(pos,b_{i},h,k+1)$ and for every
$NAF$ pre-condition $b_{j}$ we have the atom $causedBy(neg,b_{j},h,k+1)$. Then
if $k<M$, where we have $max\_graph\_lvl(M)$ for some integer value of $M$, we get the atoms $justify(b_{i},k+1)$, for every
positive pre-condition $b_{i}$ which is not a user provided fact. Finally each $causedBy$ atom generates a
$directedEdge$ atom, and these atoms are the set of directed edges
representing the justification graph.

\subsection{Some example executions}
Given the following program
\begin{lstlisting}[numbers=left]
gen_graph(relA(john)).
max_graph_lvl(5).

user_input(pos,relE(john,james,mary)).
user_input(pos,relD(john,james,mary)). 

holds(X):-user_input(pos,X).

holds(relA(X)) :- holds(relB(X, Y)),not holds(relC(X,Y)).
holds(relB(X,Y)) :- holds(relD(X,Y,Z)), holds(relE(X,Y,Z)).

causedBy(pos,relB(X,Y),relA(X),N+1):-holds(relA(X)),holds(relB(X, Y)),not holds(relC(X,Y)),justify(relA(X),N).
causedBy(neg,relC(X,Y),relA(X),N+1):-holds(relA(X)),holds(relB(X, Y)),not holds(relC(X,Y)),justify(relA(X),N).

causedBy(pos,relD(X,Y,Z),relB(X,Y),N+1):-holds(relB(X,Y)),
holds(relD(X,Y,Z)),holds(relE(X,Y,Z)),justify(relB(X,Y),N).
causedBy(pos,relE(X,Y,Z),relB(X,Y),N+1):-holds(relB(X,Y)),
holds(relD(X,Y,Z)),holds(relE(X,Y,Z)),justify(relB(X,Y),N).

justify(X,N):-causedBy(pos,X,Y,N), not user_input(pos,X),N<M, max_graph_lvl(M).
directedEdge(Sgn,X,Y):-causedBy(Sgn,X,Y,M).

justify(X,0):-gen_graph(X),not user_input(pos,X).

directedEdge(pos,userFact,X):-directedEdge(pos,X,Y), user_input(pos,X).

directedEdge(pos,userFact,X):-gen_graph(X), user_input(pos,X).
\end{lstlisting}
We get the following set of directed edges representing the justification graph. 
\begin{lstlisting}[frame=none]
directedEdge(pos,relB(john,james),relA(john))
directedEdge(pos,relE(john,james,mary),relB(john,james))
directedEdge(pos,relD(john,james,mary),relB(john,james))
directedEdge(neg,relC(john,james),relA(john))
directedEdge(pos,userFact,relD(john,james,mary))
directedEdge(pos,userFact,relE(john,james,mary))
\end{lstlisting}

%%% Local Variables:
%%% mode: latex
%%% TeX-master: "main"
%%% End:

\section{Simple Abductive Proof Generation Task}

We shall define here the notion of a
\textit{Simple Abductive Proof Generation Task}, as all of our formal results
will apply to this restricted class of abductive proof generation tasks.

\begin{definition}[Simple Abductive Proof Generation Task]\label{simpletask}
Given an abductive proof generation task $\langle R,q,U,C,N \rangle$, we say that this task is a $\textit{simple abductive proof generation task}$ if the following hold:
\begin{enumerate}
\item $R$ contains no negation as failure.
\item $R$ contains no function symbols, arithmetic operators.
\item No post condition of any rule in $R$ contains repeated variables. For example the rule:\\
$p(X,X):-r(X,X,Y)$ is not allowed but the rule: 

$p(X,Y):-r(X,X,Y)$ is allowed.
\item $C$ is empty.
\item Any constraint on abducibles in $U_{a}$ must consist of only a single
  positive fully un-ground atom with no repeated variables amongst its
  arguments. For example if $p$ is a binary predicate, then the constraint
  $:-abducedFact(p(X,Y)).$ is allowed but the constraint
  $:-abducedFact(p(X,X)).$ Constraints where more than one atom appears are
  also not allowed. For instance the following would be disallowed:
  $:-abducedFact(p(X,Y)),abducedFact(r(X)).$ Finally, constraints containing
  partially or fully ground atoms are disallowed. For example the following
  would be disallowed\\ $:-abducedFact(p(james,Y)).$
\item If $U_{a}$ is such that no instance of some predicate $p$ can be
  abduced, then $U_{f}$ must not contain any instantiation of $p$.
\item $q$ must be positive and fully ground.
\end{enumerate}
\end{definition}

%%% Local Variables:
%%% mode: latex
%%% TeX-master: "main"
%%% End:

\section{Finiteness and Completeness Properties of Simple Tasks}

\begin{theorem}[Finiteness]\label{thm:finiteness}
Assume that $\langle R,q,U,\emptyset,N \rangle$ is such that it is a simple abductive proof generation task. %$R$ is function free and
%also contains no arithmetic operations. $N$ is finite. $q$ is a ground
%positive atom containing no function symbols. $U$ contains some integrity
%constraints and ground positive atoms containing no function symbols.% Then
Then $P_{\langle R,q,U,\emptyset,N \rangle}^{res}$ cannot have infinite answer sets. 
\end{theorem}

The proof can be found in Section~\ref{sec:proof_finiteness}.

%%% Local Variables:
%%% mode: latex
%%% TeX-master: "main"
%%% End:

% \section{Completeness for simple  problems}\label{sec:completeness}

\begin{theorem}[Completeness]\label{thm:completeness}
  Given a simple abductive proof generation task, if there exists a general solution $S$ to that task then there
  exists a ASP solution $S_{ASP}$ to that task corresponding to an answer set of the ASP program $P_{\langle R,q,U,\emptyset,N\rangle}^{res}$.
\end{theorem}

The proof can be found in Section~\ref{sec:proof_completeness}.

%\section{Summary of results - I}
Let us just comment on the results above in a slightly broader context . Firstly it is not difficult to see that the above results for finiteness and completeness hold for a slightly larger class of abductive proof generation tasks than the class of simple tasks. Namely we can in fact relax condition 7 in the definition of simple tasks, to allow $q$ to be un-ground or only partially ground. Call this class of abduction tasks $semi$-$simple$.  Also, the completeness result in fact holds if $P_{\langle R,q,U,\emptyset,N\rangle}^{res}$ is replaced with $P_{\langle R,q,U,\emptyset,N\rangle}^{semi-res}$. In summary what we have then is the following\\
Given a $\textit{semi-simple}$ task $\langle R,q,U,\emptyset,N\rangle$. $P_{\langle R,q,U,\emptyset,N\rangle}^{res}$, $P_{\langle R,q,U,\emptyset,N\rangle}^{semi-res}$, both enjoy the completeness and finiteness properties. However only $P_{\langle R,q,U,\emptyset,N\rangle}^{semi-res}$ supports full implicit term substitution whereas $P_{\langle R,q,U,\emptyset,N\rangle}^{res}$ only supports partial term substitution. We shall formulate and prove a formal result regarding term substitution for $P_{\langle R,q,U,\emptyset,N\rangle}^{semi-res}$ next.    

%%% Local Variables:
%%% mode: latex
%%% TeX-master: "main"
%%% End:

\section{Proof simplification using User Provided Facts}\label{sec:proof_simplification}
\begin{definition}[Abstract Proof Graph]
Given a rule set $R$ which does not contain NAF, a predicate $p$ and integer $n$ define the abstract proof graph
$G_{R,p,n}$ as follows. The nodes of $G_{R,p,n}$ is the set of $query$
predicates generated by the rules just by the rules $AG1$ and $AG3$, in the encoding $P_{\langle R,q,U,C,N\rangle}^{res}$ where in
$\langle R,q,U,C,N\rangle$, $q$ is $p(v1,v2..,vk)$ assuming $p$ has arity $k$,
$U$, $C$ are empty and $N = n$. The edge relation is defined as follows. Two
nodes $d1$, $d2$ are connected by a directed edge represented as $E(d1,d2)$ if
and only if, $d1$ represents a pre-condition of an input rule where $d2$ is
the post condition.
\end{definition}
So if $R'$ consisted of the rules:
\begin{lstlisting}[frame=none]
a(X):-b(X,Y),c(Y).
b(X,Y):-d(X,Y,Z).
\end{lstlisting}
Then $G_{R',a,2}$ is:
\begin{lstlisting}[frame=none]
E(query(b(v1,sk(v1)),1),query(a(v1),0)),
E(query(c(sk(v1)),1),query(a(v1),0)))
E(query(d(v1,sk(v1),sk'(v1,sk(v1))),2),query(b(v1,sk(v1)),1))
\end{lstlisting}
Here $sk$ $sk'$ are just abbreviations of the full skolem function names. Also we assume the order $[X,Y,Z]$ on variables in the second rule, and the order $[X,Y]$ on variables in the first rule. (Recall that when defining the abducible generation rules in section 3 we had an order on variables in a rule)
\begin{definition}[Abstract Instance Set]\label{absinst}
Given a rule set $R$, predicate $p$ and integer $n$, the corresponding Abstract Instance Set denoted corresponding to this triple denoted by  $I_{R,p,n}$ is the set of $create\_Sub$
predicates generated by the rules just by the rules $AG1$ and $AG3$, in the encoding $P_{\langle R,q,U,C,N\rangle}^{res}$ where in
$\langle R,q,U,C,N\rangle$, $q$ is $p(v1,v2..,vk)$ assuming $p$ has arity $k$,
$U$, $C$ are empty and $N = n$. 
\end{definition}
So for the example above the set $I_{R',a,2}$ is:\\ $createSub(subInst\_r1(v1,sk(v1)),1)$,\\ $createSub(subInst\_r2(v1,sk(v1), sk'(v1,sk(v1))),2)$ 

\begin{definition}[Minimal abstract Proof Graph]\label{minabsgraph}
Given an abstract proof graph $G_{R',p,N}$, construct the minimal proof graph $G_{R',p,N}^{min}$ as follows. Firstly for given an integer $k$, going from left to right, delete all duplicate nodes $query(a',k)$ such that $query(a,k)$ is already in the proof graph and $a'=a$. Next for going down the proof graph, for each $k\in {0,1...,N}$, delete the node $query(a,k)$, if there exists $l<k$, such that $query(a',l)$ is in the proof graph $a=a'$. This forms $G_{R',p,N}^{min}$. The edge relation is inherited from $G_{R,p,N}$ in the obvious way. 
\end{definition}
Note firstly that if for some $b,j$, $query(b,j)$ is in $G_{R',p,N}$, then there exists some $h\leq j$ such that $query(b,h)$ is in $G_{R',p,N}^{min}$. Secondly we have the following property:
\begin{lemma}
Given $query(a,k)$ in $G_{R',p,N}^{min}$, unless $k=0$, there exists $query(a',k-1)$ in $G_{R',p,N}^{min}$ such that we have 

$E(query(a,k),query(a',k-1))$.  
\end{lemma}

\begin{proof} \textit{(sketch)}. \\
Let $query(a'',k-1)$ be such that $E(query(a,k),query(a'',k-1))$ is an edge of $G_{R',p,N}$. Now if $query(a'',k-1)$ is not in $G_{R',p,N}^{min}$, then there exists $s<k-1$ such that $query(a'',s)$ is in $G_{R',p,N}^{min}$. Therefore $query(a'',s)$ is in $G_{R',p,N}$. Then it follows that $query(a,s+1)$ is in $G_{R',p,N}$. However this is a contradiction since $s+1<k$. Hence it must be the case that $query(a'',k-1)$ is in $G_{R',p,N}^{min}$. 
\end{proof}

For the example we have been considering, the minimal abstract proof graph is the same as the abstract proof graph.

\begin{definition}[Concrete Proof Graph]\label{def:concretegraph}
Given a minimal abstract proof graph $G_{R,p,n}^{min}$, and a substitution $\theta$ for
terms in the minimal abstract proof graph, define the concrete proof graph
$C_{R,p,n,\theta}$ to be the set of $query$ atoms obtained after doing the
substitution $\theta$ on the set of $query$ atoms in $G_{R,p,n}^{min}$. (We drop the $min$ from the notation for the concrete proof graph as we will always mean a substitution from the minimal abstract proof graph)
\end{definition}

Note that such a substitution $\theta$ need not be injective.

\begin{definition}[Parent node, child node, sibling node, descendant]
Given any concrete or abstract, minimal or non minimal proof graph $G$ and given two nodes $d1,d2$ in $G$, we say that $d2$ is a $\textit{parent}$ of $d1$ if we have $E(d1,d2)$. In such a case we say $d1$ is a $\textit{child}$ of $d2$. Given two nodes $d3$ and $d4$, we say $d3$ is a $\textit{sibling}$ node of $d4$, if there exists some node $d'$ such that $E(d3,d')$ and $E(d4,d')$. ($d4$ is then also a sibling node of $d3$) Given two nodes $d5$ and $d6$, we say $d5$ is a $\textit{descendant}$ of $d6$ if $E_{tr}(d5,d6)$ is true where $E_{tr}$ is the transitive closure of $E$. 
\end{definition}

\begin{definition}[Concrete Instance Set]
Given an abstract instance set $I_{R,p,n}$, and a substitution $\theta$ for
terms in the instance set, define the concrete instance set
$I_{R,p,n,\theta}$ to be the set of $createSub$ atoms obtained after applying the
substitution $\theta$ on the set of $createSub$ atoms in $I_{R,p,n}$, (which itself is associated to the full abstract proof graph rather than the minimal one).
\end{definition}

For our example consider the substitution $\theta =[v1\rightarrow john,sk(v1) \rightarrow extVar, sk'(v1,sk(v1))\rightarrow extVar]$. Then $C_{R',a,2,\theta}$ is :
\begin{lstlisting}[frame=none]
E(query(b(john,extVar),1),query(a(john),0)),
E(query(c(extVar),1),query(a(john),0)))
E(query(d(john,extVar,extVar),2),query(b(john,extVar),1))
\end{lstlisting}
$I_{R,p,n,\theta}$ = $createSub(subInst\_r1(john,extVar),1)$,\\ $createSub(subInst\_r2(john,extVar, extVar),2)$

\begin{definition}[Derived Substitution - T ]\label{def:derivesub}
Given some $G_{R,p,n}^{min}$, and an associated $C_{R,p,n,\theta}$, let $q_{c}$ be a $query$ atom from $C_{R,p,n,\theta}$. Let $S_{q_{c}}$ be the set of $query$ atoms in $G_{R,p,n}^{min}$, which upon application of the substitution $\theta$ give $q_{c}$. Now pick an atom $q_{o}$ from $S_{q_{c}}$. Now suppose $q_{c}$ is given by $query((t_{1},t_{2},...,t_{j}),k)$, $q_{o}$ is given by $query((e_{1},e_{2},...,e_{j}),k)$. Now consider an arbitrary query atom $q_{f}$ given by $query((a_{1},a_{2},...,a_{j}),k)$, which is such that the map $\psi$ mapping each $e_{i}$ to the corresponding $a_{i}$ is well defined. (ie. $\psi$ is not one-to-many). Then define the substituion $\phi = T(\theta, q_{c},q_{o},q_{f})$ on terms in $G_{R,p,n}^{min}$ by the following:\\ For a term $u$ in $G_{R,p,n}^{min}$ (meaning $u$ occurs as the argument of the predicate inside some $query$ atom), if $u$ is in the set $\{e_{1},e_{2},...,e_{j}\}$ then $\phi(u) = \psi(u)$, otherwise $\phi(u) = \theta(u)$.
\end{definition}
For instance going back to our example if we let $q_{c}$ be\\
$query(b(john,extVar),1)$ let $q_{o}$ be $query(b(v1,sk(v1)),1)$ and let $q_{f}$ be $query(b(john,james),1)$ then $\phi=T(\theta,q_{c},q_{o},q_{f})$ is the substitution $\theta =[v1\rightarrow john,sk(v1) \rightarrow james, sk'(v1,sk(v1))\rightarrow extVar]$. We now have the following theorem:

\begin{theorem}[Term substitution]\label{thm:termsub}
Consider the abductive proof generation task $\langle R,q,U,C,N\rangle$, and suppose that his task is a simple abductive proof generation task. Let $p$ be the predicate corresponding to $q$. Suppose $A$ is an answer set of $P^{semi-res}_{<R,q,U,C,N>}$ and let $A$ contain $C_{R,p,N,\theta}$ and $I_{R,p,N,\theta}$ . Now
say the atom $q_{c}$ $query(p_{i}(t_{1},t_{2},..,t_{j}),k)$ is in
$C_{R,p,N,\theta}$. Let $q_{o}$ be some query atom from the set $S_{q_{c}}$
given by $query((e_{1},e_{2},...,e_{j}),k)$. Now suppose, $q_{f}=
query((a_{1},a_{2},...,a_{j}),k)$ is an arbitrary $query$ atom such that the
map $\psi$ from the $e_{i}s$ to the $a_{i}s$ as described in the definition
above is well defined. Then upon adding the fact
$query(p_{i}(a_{1},a_{2},..,a_{j}),k)$ to $P^{semi-res}_{<R,q,U,C,N>}$,
the resulting program has an answer set $A'$ such that $A'$ contains
$C_{R,p,N,\phi}$ and $I_{R,p,N,\phi}$  where $\phi = T(\theta, q_{c}, q_{o}, q_{f})$
\end{theorem} 

The proof can be found in Section~\ref{sec:proof_term_substitution}.

\begin{corollary}[Adding facts]\label{thm:addfact}
Given the simple abductive proof generation task $\langle R,q,U,C,N\rangle$, let $p$ be the predicate corresponding to $q$. Suppose $A$ is an answer set of $P^{semi-res}_{<R,q,U,C,N>}$ and let $A$ contain $C_{R,p,N,\theta}$ and $I_{R,p,N,\theta}$. Now say the atom $q_{c}$ $query(p_{i}(t_{1},t_{2},..,t_{j}),k)$ is in $C_{R,p,N,\theta}$. Let $q_{o}$ be some query atom from the set $S_{q_{c}}$ (The set of pre-images of $q_{c}$ in $G_{R,p,N}^{min}$) given by $query((e_{1},e_{2},...,e_{j}),k)$. Now suppose, $h_{f}= holds((a_{1},a_{2},...,a_{j}))$ is an arbitrary $holds$ atom such that the map $\psi$ from the $e_{i}s$ to the $a_{i}s$ as described in the earlier definition of derived substitution is well defined. Then upon adding the fact $holds(p_{i}(a_{1},a_{2},..,a_{j}))$ to $P^{semi-res}_{<R,q,U,C,N>}$, the resulting program has an answer set $A'$ such that $A'$ contains $C_{R,p,N,\phi}$ and $I_{R,p,N,\phi}$ , where $\phi = T(\theta, q_{c}, q_{o}, q_{f})$, where $q_{f} = query(p_{i}(a_{1},a_{2},..,a_{j}),k)$. 
\end{corollary}

The proof can be found in Section~\ref{sec:proof_term_substitution}.

The preceding theorem and corrolary correspond to the notion of full implicit term substitution which we discussed earlier. For example going to our main example in this section. Let $R$ be the set of rules:\begin{lstlisting}[frame=none]
a(X):-b(X,Y),c(Y).
b(X,Y):-d(X,Y,Z).
\end{lstlisting}

Let $q$ be $a(john)$, let the set of user provided facts be empty and suppose that no instances of the predicate $a$ or $b$ can be abduced. Finally suppose that the set $C$ is empty and $N=2$. Then upon running the ASP program $P_{\langle R,q,U,C,N\rangle}^{semi-res}$ we will get the minimal abductive solution \begin{lstlisting}[frame=none]
abducedFact(c(extVar)), abducedFact(d(john,extVar,extVar))
\end{lstlisting} 
This corresponds to the substitution $\theta =[v1\rightarrow john,sk(v1) \rightarrow extVar, sk'(v1,sk(v1))\rightarrow extVar]$. Then upon modifying the set of user provided facts by adding the fact $c(james)$, we get the smaller abductive solution:
\begin{lstlisting}[frame=none]
abducedFact(d(john,james,extVar))
\end{lstlisting} which corresponds to the substitution $\theta =[v1\rightarrow john,sk(v1) \rightarrow james, sk'(v1,sk(v1))\rightarrow extVar]$. 

% \section{Summary of results - II}
As was the case for the finiteness and completeness results, it is in fact the case the theorem and corollary proved in the previous section hold for the slightly larger class of abductive proof generation tasks which we called semi-simple. Overall, we have the following results. Given a semi-simple task $\langle R,q,U,\emptyset,N\rangle$, if no rule in $R$ contains existential variables, then to solve this task we can use the encoding $P_{\langle R,q,U,\emptyset,N\rangle}^{exp}$, which enjoys the completeness, finiteness and full term substitution properties. If $R$ does contain existential variables then we can either use the encoding $P_{\langle R,q,U,\emptyset,N\rangle}^{res}$ which enjoys the properties of completeness and finiteness but only gives us partial term substitution or we can use $P_{\langle R,q,U,\emptyset,N\rangle}^{semi-res}$ which enjoys the properties of completeness, finiteness and full term substitution.

%%% Local Variables:
%%% mode: latex
%%% TeX-master: "main"
%%% End:

\section{ Conclusions}\label{sec:conclusion}
We have presented several encodings for abductive proof generation in ASP,
incorporating notions of depth control and  novel implementations of term
substitution. We have also given an encoding that allows one to generate a set
of directed edges representing a justification graph.

It seems to us that some of the ideas involved in the term substituion
mechanism are similar to the ideas involved when one uses $\textit{Sideways
  Information Passing Strategies}$ \cite{beeri91} to re-write datalog rules
for more efficient evaluation of queries by incorporating elements of top-down
reasoning. However we have not explored this connection in detail. Those techniques 
typically involve a complete re-write of the input
rules according some chosen fixed sideways information passing strategy, which
makes that whole approach quite different to ours. \cite{DBLP:journals/jar/Stickel94} describes an approach 
to doing abductive reasoning in a bottom up manner. He uses 'continuation predicates' to pass substitutions from previously evaluated rule pre-conditions to rule pre-conditions yet to be evaluated, given the rule post-condition as the 'goal'. This somewhat resembles our use of 'createSub' predicates. However it seems to us that
that approach 
imposes a strict order on the evaluation of preconditions of a rule, which makes 
that method much less general than ours. 

There are several possible directions for future work. One possible line of theoretical investigation 
could be to study how the abductive solutions calculated by our methods (and any resulting extra
consequences of the abductive solution and input rules) could be generalised to 
sentences in first order logic. Roughly speaking, given an abductive solution involving
instances of 'extVar' the aim would be to map these solutions to solutions where instances
of 'extVar' are replaced by universally quantified variables (where perhaps such a variable may
not take values from some finite set).  Distinguishing between instances of 'extVar' that should 
get mapped to distinct universally quantified variables can be done by adding certain facts 
that would result in the generated abductive solution being modified so that the 'matching' 
occurrences of 'extVar' get replaced by some other fresh constant. Intuitively, it seems to 
us that our method of calculating and simplifying abductive solutions without grounding over
the entire domain of constants gives an appropriate 
setting to explore some of these ideas. Of course the correctness/applicability of such techniques 
would have to be investigated in a rigorous and formal manner, see
Section~\ref{sec:generalisation_sol} for a further discussion.

Another possible future line of work may include extending the formal results
presented here to a larger class of abductive proof generation problems. 
It also seems to us that the main technique used to generate the directed edge
set representing the justification graph could be adapted for use in SAT/SMT
solvers to get justifications out of them. 

We also have yet to study the complexity problems
associated with the methods presented in this paper. \cite{DBLP:journals/tcs/EiterGL97} provides a thorough study 
of the complexity of abductive reasoning. It remains to be seen how those results, 
many of which deal with propositional logic, could be carried over to our setting, where
we aim to compute abductive solutions without complete grounding of rules. %\remms[inline]{Also mention Haskell implemenation}

%%% Local Variables:
%%% mode: latex
%%% TeX-master: "main"
%%% End:

\begin{acks}
This research is supported by the National Research Foundation (NRF),
Singapore, under its Industry Alignment Fund -- Pre-Positioning Programme, as
the Research Programme in Computational Law. Any opinions, findings and
conclusions or recommendations expressed in this material are those of the
authors and do not reflect the views of National Research Foundation,
Singapore.
\end{acks}
  
%----------------------------------------------------------------------
%\bibliographystyle{alpha}
\bibliographystyle{ACM-Reference-Format}
\bibliography{main_extended}

% ----------------------------------------------------------------------
\newpage
\appendix
\section{Extra Material}

\subsection{Proof of Finiteness}\label{sec:proof_finiteness}

We will prove the finiteness result of Theorem~\ref{thm:finiteness} via a series of lemmas. Also we will refer to specific lines of the encoding given below but the arguments in the proof are fully general and can be easily extended to the encoding of any abductive proof generation problem

\begin{lstlisting}[numbers=left]
max_ab_lvl(5).
query(relA(bob),0).
goal:-holds(relA(bob)).
:-not goal.


holds(relA(P)) :- holds(relB(P, R)),holds(relD(R)).
holds(relB(P, R)) :- holds(relA(R)), holds(relC(P)).

explains(relB(P, R), relA(P) ,N) :- createSub(subInst_r1(P,R),N).
explains(relD(R), relA(P) ,N) :- createSub(subInst_r1(P,R),N).


createSub(subInst_r1(P,skolemFn_r1_R(P)),N+1) :- query(relA(P) ,N),max_ab_lvl(M),N<M-1.
createSub(subInst_r2(P,Q),N+1) :- query(relB(P, Q) ,N),max_ab_lvl(M),N<M-1.


explains(relA(R), relB(P,R) ,N) :- createSub(subInst_r2(P,R),N).
explains(relC(P), relB(P,R) ,N) :- createSub(subInst_r2(P,R),N).



createSub(subInst_r1(P,R),M-1) :- createSub(subInst_r1(V_P,V_R),N), holds(relB(P, R)),max_ab_lvl(M).
createSub(subInst_r1(V_P,R),M-1) :- createSub(subInst_r1(V_P,V_R),N), holds(relD(R)),max_ab_lvl(M).

createSub(subInst_r2(V_P,R),M-1) :- createSub(subInst_r2(V_P,V_R),N), holds(relA(R)),max_ab_lvl(M).

createSub(subInst_r2(P,V_R),M-1) :- createSub(subInst_r2(V_P,V_R),N), holds(relC(P)),max_ab_lvl(M).


query(X,N):-explains(X,Y,N),max_ab_lvl(M),N<M.
{abducedFact(X)}:-query(X,N).
holds(X):-abducedFact(X).
holds(X):-user_input(pos,X).
\end{lstlisting}

\begin{lemma}
Let $S$ be the set of uninterpreted skolem functions appearing in the abducible
generation encoding. Let $C$ be the set of constants occurring in either a user
inputed fact in $U$ or the query $q$. Let $T_{k}$ denote the set of unique
terms that can be constructed from $S$ and $C$ with skolem depth at most
$k$. Let $T$  be the possibly infinite set consisting of terms of unbounded
depth. Then for any positive integer $k$, $T_{k}$ is finite.
\end{lemma}
% \remms{Better write: $T$ is the union of the $T_k$. However, the lemma does not say anything about $T$} 

\begin{proof}
This is a standard result which follows easily by induction. Clearly $|T_{0}|$ = $|C|$, which is finite. Let $|T_{k}|$ = $B$ be finite. Let $S$ be a set of $l$ functions having arity at most $d$. Then $|T_{k+1}|\leq B + lB^{d}$. Hence $T_{k+1}$ is finite.  
\end{proof}

\begin{lemma}
 For any predicate $p$ inside an $abducedFact$ atom the arguments of $p$ are elements of the set $T_{N}$. Similarly for any predicate $p$ inside an $holds$ atom the arguments of $p$ are elements of the set $T_{N}$. Similarly for the other atoms in any answer set of $P_{\langle R,q,U,C,N \rangle}^{res}$. Terms that correspond to arguments of predicates occuring in the input rules are always elements of $T_{N}$.
\end{lemma}

\begin{proof}\textit{(sketch)}. Firstly note that for any predicate $p$ inside a
$abducedFact$ or $holds$ atom, where $p$ comes from the input rule set, the
arguments of $p$ belong to the (possibly infinite) set $T$. We can see this by observing that
since no rule in $R$ contains a function symbol, the encoding of the rules
themselves such as in line 7,8 of the ASP program above, cannot introduce new
terms in the arguments of predicates inside $holds$ atoms. Similarly lines
like 23, 24 etc. also cannot introduce new terms inside $holds$, $create\_sub$
atoms in the final answer set, which means also that no new terms appear as
predicate arguments inside $abducedFact$ atoms.

Only lines like 14,15 are able to create fresh terms that become arguments of
various atoms in the answer set. But then it follows that any $holds$ atom,
$create\_sub$, $abducedFact$, $query$, $explains$ atoms can only have input
rule predicate arguments from the set $T$.

Next we show that the skolem depth of these terms is at most $N$. Any fresh
term $t$ from $T$ appearing in an atom in an answer set of the ASP program,
such that $t$ does not appear in any user provided fact or in the original
query $q$, must have been constructed from rules like in lines 14,15. But
these encode a 'static' space of abducibles, which unlike lines 23, 24 is
independent of the forward reasoning component or any user provided additional
facts. This is because any $query$ atom that gives an instantiation of the
right hand side of rules like the one in line 14, must have integer argument
less than $M-1=N$, (recall that we have $M=N+1$), whereas rules like the one
in line 24, gives us $query$ atoms with integer argument $M-1$. It is not hard
to see that because of this fresh terms created via ASP rules like the one on
line 14 must have skolem depth at most $N$. Therefore all terms inside atoms
of the answer-set that correspond to input rule predicate arguments must
belong to the set $T_{N}$.
\end{proof}

\begin{lemma}
Since the set $T_{N}$ is finite it follows that any answer set of $P_{\langle
  R,q,U,C,N \rangle}^{res}$ is finite.
\end{lemma}

\begin{proof}
Since $T_{N}$ is finite and the integer argument of any $create\_sub$, $explains$, $query$ 
atom is bounded by $N$, any atom in the final answer set has only finitely
many instantiations. This proves finiteness of the resulting answer set.
\end{proof}

\subsection{Proof of Completeness}\label{sec:proof_completeness}

We now give a proof of Theorem~\ref{thm:completeness}.

Note that without loss of generality, for the sake of proving completeness, we may assume that the set of user supplied facts is empty. We will first need a preliminary lemma.

\begin{lemma}
Assume $\langle R,q,U,\emptyset,N\rangle$, $\langle R,q',U,\emptyset,N\rangle$ are two
problems satisfying the conditions above but $q'$ is obtained from $q$ by
possibly changing some or all of the arguments of $q$. Then $\langle
R,q,U,\emptyset,N\rangle$ has a solution $S_{ASP}$ derived from the ASP program $P_{\langle R,q,U,\emptyset,N\rangle}^{res}$ if and only if $\langle R,q',U,\emptyset,N\rangle$ has a solution $S'_{ASP}$ derived from the ASP program $P_{\langle R,q',U,\emptyset,N\rangle}^{res}$ We call such solutions ASP solutions for short.
\end{lemma}

\begin{proof}
We prove the lemma by induction on $N$.
The case $N = 0$ is trivial. If there exists a ASP solution for $\langle
R,q,U,\emptyset,0\rangle$ then any instance of the predicate in $q$ can be
abduced therefore $\langle R,q',U,\emptyset,0\rangle$ has a ASP solution. Assume
the result for all $N<k$ such that $k>0$. A non trivial ASP solution to
$\langle R,q,U,\emptyset,k+1\rangle$ (ie a solution in which no instance of the predicate in $q$ can be abduced) implies the existence of some ASP solution to
each of the following set of problems $\langle R,q_{pre_{i}},U,\emptyset,k\rangle$
where the $q_{pre_{i}}$ are a set of
pre-conditions of a rule $r$ in $R$, under some substitution $\theta$ where
$\theta$ applied to the post-condition of $r$ gives $q$. But then since the post- condition of $r$
contains no repeated variables, if $q$ unifies with the post condition of $r$
under some substituion $\theta$, then there exists some other substitution
$\theta'$ for the variables in $r$ such that $\theta'$ applied to the
postcondition of $r$ gives $q'$. Therefore the union of ASP solutions to the set of problems
$\langle R,q'_{pre_{i}},U,\emptyset,k\rangle$ is also a ASP solution to $\langle
R,q',U,\emptyset,k+1\rangle$ where each $q'_{pre_{i}}$ is obtained from
$q_{pre_{i}}$ by possibly changing the arguments in the predicate. By the
induction hypothesis an ASP solution to each of the problems $\langle
R,q'_{pre_{i}},U,\emptyset,k\rangle$ exists. Hence by taking the union of the solutions we get an ASP solution to  $\langle
R,q',U,\emptyset,k+1\rangle$. This proves the lemma.
\end{proof}

We now prove the main theorem by induction on $N$. 

\begin{proof}
The case $N=0$ is trivial. Assume
the result for all $N<k$, where $k>0$. The existence of a non-trivial general solution
to $\langle R,q,U,\emptyset,k+1\rangle$ implies the existence of a solution to each of the following set of problems, $\langle R,q_{pre_{i}},U,\emptyset,k\rangle$ where the $q_{pre_{i}}$ are the set of
pre-conditions of a rule $r$ in $R$, under some substitution $\theta$ where
$\theta$ applied to the post-condition of $r$ gives $q$. But then by the
inductive hypothesis there is an ASP each of the same set of problems. By
the lemma proved above there exists an ASP solution to each of the following set of
problems $\langle R,q'_{pre_{i}},U,\emptyset,k\rangle$, where each
$q'_{pre_{i}}$ is obtained from the corresponding $q_{pre_{i}}$, by possibly
replacing some predicate arguments with the appropriate skolem terms for
predicate arguments that correspond to existential variables in the
preconditions of $r$. Hence taking a union of these ASP solutions we get an ASP solution to
$\langle R,q,U,\emptyset,N\rangle$.
\end{proof}

\subsection{Proof of Term Substitution}\label{sec:proof_term_substitution}

Proof of Theorem~\ref{thm:termsub}:

\begin{proof} \textit{(sketch)}. In the following proof when we refer to the $level$ of a $query$ atom we mean the integer argument of the $query$ atom. When we refer to the $predicate$ $argument$ of a $query$ atom we will mean the first argument of a $query$ atom. Given a $query$ atom $q$ in $G_{R,p,N}^{min}$, the transformed atom $tr_{\chi}(q)$ denotes the corresponding atom in $C_{R,p,N,\chi}$, given some substitution $\chi$ of terms in $G_{R,p,N}^{min}$. We will prove the result by induction on $k$, the level of the atoms $q_{c}$, $q_{o}$ and $q_{f}$. 

Fix $q_{c}$, $q_{o}$ and $q_{f}$ and suppose that the level of these atoms $k$ is $0$. Since there is only one level $0$ query atom say $q'$ in $G_{R,p,N}$, we must have $q_{o}=q'$, and $q_{f} = tr_{\phi}(q')$. Now we have to show that for all $query$ atoms $q$ in $G_{R,p,N}$ such that the level of $q$ is greater than $0$, we have $tr_{\phi}(q)$. Now suppose that $p$ is a 3 place predicate and there is an input rule $r_{c}$ where ${c}$ refers to the rule id. Suppose that the predicate $p_{d}$ appears as a pre-condition in $r_{c}$. Let the variables in $r_{c}$ be $X1$, $X2$, $X3$, $X4$, in the order given by $O_{c}$, (recall from Section~\ref{sec:derived_asp} that for each input rule we had some order on the variables). Suppose $r_{c}$ has the following form:
\begin{verbatim}
p(X2,X1,X4):-p_d(X4,X2,X3),...    
\end{verbatim}

Then in $C_{R,p,N,\theta}$, we have the following $query$ and $createSub$ atoms:

$query(p(\theta(v1),\theta(v2),\theta(v3),0)$,\\ $query(p_{d}(\theta(v3),\theta(v1),\theta(sk(v2,v1,v3)),1))$. 

Here $sk$ stands for the skolem function $skolemFn\_c\_x3$. 

We also have\\
$createSub(subInst\_r_{c}(\theta(v2),\theta(v1),\theta(sk(v2,v1,v3)), \theta(v3)),1)$. Note that $\theta(v1)=t_{1},\theta(v2)=t_{2},\theta(v3)=t_{3}$, furthermore\\ $q_{o}= query(p(v1,v2,v3),0)$. Now given $q_{f} = query(p(a_{1},a_{2},a_{3}),0)$ as above we consider the following $AG2$ clause:
\begin{lstlisting}[frame=none]
createSub(subInst_r_c(X1,X2,V_X3,X4),N):-
createSub(subInst(V_X1,V_X2,V_X3,V_X4),N),query(p(X2,X1,X4),L).    
\end{lstlisting}
We have the following instantiation of the clause above

$createSub(subInst\_r\_c(a_{2},a_{1},\theta(sk(v2,v1,v3)),a_{3}),1):-$  

$createSub(subInst\_r\_c(t_{2},t_{1},\theta(sk(v2,v1,v3)),t_{3}),1)$,\\ $query(p(a_{1},a_{2},a_{3}),0).$\\
Then via an instantiation of the following $AG1$ clause:
\begin{lstlisting}[frame=none]
explains(p_d(X4,X2,X3),p(X1,X2,X4),N):-
createSub(subInst_r_c(X1,X2,X3,X4),N).    
\end{lstlisting}
and an instantiation of the following $AG1$ clause:
\begin{lstlisting}[frame=none]
query(X,N):-explains(X,Y,N). 
\end{lstlisting}

We get the following atom $query(p_{d}(a_{3},a_{1},\theta(sk(v2,v1,v3)),1)$ which is $tr_{\phi}(query(p_{d}(t_{3},t_{1},\theta(sk(v2,v1,v3))),1))$. Similarly for any other child node $q$ of $query(p(v1,v2,v3),0)$ in $G_{R,p,N}^{min}$, we get $tr_{\phi}(q)$. Now it follows by a similar argument to before that for any child node $q'$ of these transformed level one nodes $q$, we also get $tr_{\phi}(q')$. One can see this by considering the following: Given a level one node $q$ from $G_{R,p,N}^{min}$, let $q_{c}^{1}$ be the image of $q$ under $\theta$, let $q_{o}^{1}$ be $q$ and let $q_{f}^{1}$ be $tr_{\phi}(q)$, then consider $\phi'$ = $T(\theta, q_{c}^{1},q_{o}^{1},q_{f}^{1})$. Then on any term $b$ in $q$, such that $b$ is in the set $\{v1,v2,v3\}$, $\phi'(b)$ =$\phi(b)$ and on all other terms $t$ in $G_{R,p,N}^{min}$, $\phi'(t)=\theta(t)$. Now, given the level one node $q$ in $G_{R,p,N}^{min}$, let $q'$ be a child node of $q$, then for any term $t'$ in $q'$, such that $t'$ occurs in the set $\{v1,v2,v3\}$, it must be the case that $t'$ occurs in $q$. Therefore, the result of applying $\phi'$ on $q'$ will in fact give us $tr_{\phi}(q')$. In this way one can see that the term substitution given by $\phi$ will propogate all the way downwards over elements of $C_{R,p,N,\theta}^{min}$. This completes the case $k=0$. 

Now suppose we have proven the theorem for all values of $k<e$. Now let $q_{c},q_{o},q_{f}$ be such that the level of all these atoms in $e+1$, and let $q_{o}$ be such that $q_{o}$ does not correspond to a pre-condition of an input rule, where the pre-condition contains an existential variable. Now, let $q_{d}$ be a parent node of $q_{o}$. Let $r_{h}$ be the relevant rule and let the predicate corresponding to $q_{o}$ be $p_{z}$ and let the predicate corresponding to $q_{d}$ be $p_{y}$. Then we have the following instantiation of an $AG2$ rule 
\begin{lstlisting}[frame=none]
createSub(subInst_r_h(W''),o+1):-
query(p_z(W_f),o+1),createSub(subInst_r_h(W'),o+1).
\end{lstlisting}
where $query(p_{z}(W_{f}),o+1)$ = $q_{f}$, $W'$ is $\theta$ applied to the relevant set of terms $W$ from the abstract proof graph, and $W''$ is $\phi(W)$. Now due an instantiation of the following rule:
\begin{lstlisting}[frame=none]
explains(p_z(W_f),p_y(U),e+1):-createSub(subInst_r_h(W''),e+1).
\end{lstlisting}
and the following $AG3$ rule:
\begin{lstlisting}[frame=none]
query(Y,N-1):-explains(X,Y,N). 
\end{lstlisting}
we get the atom $tr_{\phi}(q_{d})$. Now consider
$\phi'= T(\theta, \theta(q_{d}), q_{d}, tr_{\phi}(q_{d}))$, then $\phi'=\phi$
on the terms in $q_{d}$, that also appear in $q_{o}$, and $\phi' = \theta$ on
all other terms. However all the terms that appear in $q_{o}$ also appear in
$q_{d}$, since we assumed that $q_{o}$ did not have terms corresponding to
existential variables. Also the level of $q_{d}$ is $e$. Hence by the
inductive hypothesis we are done.\\
Now, let $q_{c},q_{o},q_{f}$ be such that
some terms in the first argument of $q_{o}$ correspond to existential
variables. That is these terms are do not appear in any parent node of $q_{o}$
and are skolem functions whose input consists of terms in the parent nodes of
$q_{o}$. Let $V_{univ,q_{o}}$ consist of the set of terms in the first
argument of $q_{o}$ that correspond to universally quantified variables and
let $V_{ext,q_{o}}$ consist of the set of terms in the first argument of
$q_{o}$ that correspond to existentially quantified variables. That is, given any term in $V_{ext,q_{o}}$, no parent node of $q_{o}$ contains this term. Given some term in $V_{univ,q_{o}}$, there is a parent node of $q_{o}$ that contains this term.   Firstly, consider $\phi'$ such that $\phi'=\phi$ on the terms in $V_{univ,q_{o}}$ and
$\phi' =\theta$ on all other terms in the abstract proof graph. Then given
this $\phi'$, note that we already get the concrete proof graph
$C_{R,p,n,\phi'}$, this is because, the substitution on terms in $V_{univ,q_{o}}$ is passed to parent nodes of $q_{o}$ whose level is $e$, and hence by the inductive hypothesis this substitution of terms is passed on to all the nodes in the proof graph. So in fact to
prove the case where the level of $q_{c},q_{o},q_{f}$ is $e+1$, we can now
assume WLOG, that $q_{c_i}\neq q_{f_i}$ implies that $q_{o_i}$
corresponds to an existential variable where here $q_{c_i}$ denotes the
$i^{th}$ entry of the predicate argument of $q_{c}$ and similarly for the
others. So assume now that we are in this case and
$\phi= T(q_{c},q_{o},q_{f})$. Let $V'_{ext,q_{o}}$ be the set of terms in the
first argument of $q_{o}$ that correspond to existential variables. Now let
$q_{d}$ be the parent node of $q_{o}$. Then due to the appropriate $AG2$ and
$AG3$ clauses, it follows that for all the child nodes $q''$ of $q_{d}$, we
have $tr_{\phi}(q'')$. Then, due to a similar argument to the one we used for the
case $k=0$, it follows that for any descendant of $q_{m}$ of $q_{d}$, we will
get $tr_{\phi}(q_{d})$. Now we claim that in the minimal abstract proof graph, the
only nodes whose predicate entry contains terms from $V'_{ext,q_{o}}$ are
descendants of $q_{d}$. Given a term $t$ from $V'_{ext,q_{o}}$, occurring in some
node $j$ of the minimal abstract proof graph, since $t$ is not in the set
$\{v_{1},v_{2},..,v_{n}\}$, there exists some ancestor $j'$ of $j$ such that
$t$ corresponds to an existential variable in $j'$. Let $j_{d}$ be the parent
node of $j'$, such that $j_{d}$ does not contain the term $t$ in its predicate argument. Now, let $t$ be of the form
$skolemFn\_l\_b(g_{1},..,g_{m})$. Here $l$ refers to a input rule, $b$ refers
to some existential variable among the pre-conditons of $l$, and
$\{g_{1},g_{2},...,g_{m}\}$ refers to a fixed permutation of the arguments of
the post-condition of rule corresponding to $l$. However, from this we can
uniquely determine what the predicate argument $j_{p}$ must be. Based on the order
$O_{l}$, the arguments of the relevant instantiation of the post-condition of
$l$ are given by some fixed permutation $\pi_{O_{l}}$ of
$\{g_{1},g_{2},...,g_{m}\}$. Therefore it follows that the first argument of
$j_{p}$ is in fact the same as that of $q_{p}$. But now since we are working
with $\textit{minimal}$ abstract proof graphs it follows that
$j_{p} = q_{p}$. This proves the claim. \\It now remains to show that we also have $I_{R,p,N,\phi}$. First note that in the preceeding part of the proof, whenever we used an Abducibles Generation rule to show how a transformed query atom $q$ gives us a transformed child atom, we used invoked rules of the following from:
$createSub(t',i):-createSub(t,i),q.$\\
$explains(g',g,i):-createSub(t',i).$\\
$query(g',i):-explains(g,g',i).$\\
Notice that the $createSub$ atom in the right hand side of the first clause is always from $I_{R,p,N,\theta}$. The same holds for a transformed $query$ atom leading to the transformation of a sibling atom and for transformation of a parent atom we have:\\
$createSub(t',i):-createSub(t,i),q.$\\
$explains(g,g',i):-createSub(t',i).$\\
$query(g',i-1):-explains(g,g',i).$\\
In each case the $createSub$ atom in the right hand side of the first clause is always from $I_{R,p,N,\theta}$ and a combination of such $query$ atom transformations creates the new concrete proof graph $C_{R,p,N,\phi}$. So in order to show that we have $I_{R,p,N,\phi}$, it is justified to first assume that we have $C_{R,p,N,\phi}$, $C_{R,p,N,\theta}$ and $I_{R,p,N,\theta}$. However this is now easy to see. given an atom $createSub(t,i)$ from $I_{R,p,N,\theta}$, the following instantions of $AG$ clauses shows how one gets the corresponding $createSub(t',i)$ atom from $I_{R,p,N,\phi}$.\\
$createSub(t_{1},i):-createSub(t_{0},i),tr_{\phi}(q_{b_{1}}).$\\
$createSub(t_{2},i):-createSub(t_{1},i),tr_{\phi}(q_{b_{2}}).$\\ ...\\
$createSub(t_{l},i):-createSub(t_{l-1},i),tr_{\phi}(q_{b_{l}}).$\\
$createSub(t_{l+1},i):-createSub(t_{l},i),tr_{\phi}(q_{h}).$\\
Here $createSub(t_{0},i) = createSub(t,i)$ and $createSub(t_{l+1},i) = createSub(t',i)$. $q_{b_{1}}$, $q_{b_{2}}$..., refer to query atoms in $C_{R,p,N,\theta}$  such that the predicate arguments of these $query$ atoms correspond to preconditions of the input rule instantiation corresponding to $createSub(t,i)$. (Note that the level of these atoms need not be $i$.) Finally $q_{h}$ refers to the post condition of the input rule instantiation given by $createSub(t,i)$. Hence we do indeed have $I_{R,p,N,\phi}$.
\end{proof}

Proof of Corollary~\ref{thm:addfact}:

\begin{proof} Given $q_{o}$, let $i_{o}$ be from the set $I_{R,p,N}$ such that $q_{o}$ corresponds to some input-rule precondition corresponding to $i_{o}$. Then given $q_{c}$, and $tr_{\theta}(i_{o})$, then consider the following instantiation of an $AG2$ clause. $tr_{\phi}(i_{o}):-tr_{\theta}(i_{o}),h_{f}$. Then due to an application of the $explains(...):-createSub(...)$ clause and the use of the following $AG3$ rule:
$query(X,N):-explains(X,Y,N)$, we get the required atom $q_{f}$, where the
predicate argument of $q_{f}$ is $h_{f}$ and the level of $q_{f}$ is the same as
$q_{c},q_{o}$. So then the theorem above applies. In the case where $q_{o},
i_{o}$ are such that $q_{o}$ corresponds to a post-condition, we invoke the
$AG3$ rule $query(Y,N-1):-explains(X,Y,N)$ instead.
\end{proof}

%%% Local Variables:
%%% mode: latex
%%% TeX-master: "main"
%%% End:

%%% Local Variables:
%%% mode: latex
%%% TeX-master: "main"
%%% End:

\subsection{Generalisation of Abductive Solutions}\label{sec:generalisation_sol}

Here we shall informally discuss the idea mentioned in the conclusion, for
generalising abductive solutions given by the ‘semi-res’ encoding to first
order logic. We only present an informal discussion here, an investigation of
formal results corresponding to these ideas is left for future work. Note
firstly that we use the ‘semi-res’ encoding because it is has the finiteness
and full term substitution properties. Consider the abductive proof generation
problem given in 3.4.3. The most general solution to that problem is:

\begin{lstlisting}[frame = none] 
abducedFact(relC(john,Y)), 
abducedFact(relD(john,Y,Z)), 
abducedFact(relE(john,Y,Z))
\end{lstlisting}
where $Y$, $Z$ are variables that can be replaced by any constant. Let us see how we can obtain this solution using the encoding. Starting off with no user inputed facts, we get the unique optimal solution
\begin{lstlisting}[frame = none]
abducedFact(relC(john,extVar)), 
abducedFact(relE(john,extVar,extVar)),
abducedFact(relD(john,extVar,extVar))
\end{lstlisting}
Now our procedure is to add facts using fresh constants so that ‘corresponding’ instances of ‘extVar’ get replaced in the abductive solution. We keep doing this until there are no more occurrences of ‘extVar’ in the optimal abductive solution produced.\\ If we add the fact 
\begin{lstlisting}[frame = none]
relC(john,v1)
\end{lstlisting}
We get the unique optimal solution 
\begin{lstlisting}[frame = none]
abducedFact(relE(john,v1,extVar)),
abducedFact(relD(john,v1,extVar))
\end{lstlisting}
Now we upon further adding the fact 
\begin{lstlisting}[frame = none]
relD(john,v1,v2)
\end{lstlisting}
We get the unique minimal abductive solution 
\begin{lstlisting}[frame = none]
abducedFact(relE(john,v1,v2))
\end{lstlisting}
Note that we used a new fresh constant ‘v2’ in our second user provided fact.
Now since there are no more instances of ‘extVar’ to be replaced away, the process ends and we can derive the general solution to the original abductive problem to be 
\begin{lstlisting}[frame = none]
abducedFact(relC(john,Y)), 
abducedFact(relD(john,Y,Z)), 
abducedFact(relE(john,Y,Z)) 
\end{lstlisting}
which is in fact the most general solution. Note that critical to this procedure is the fact that only certain instances of ‘extVar’ can be replaced by terms from user provided facts. For instance in the second step after the first fact has been added, the solver will not produce 
\begin{lstlisting}[frame = none]
abducedFact(relE(john,v1,v1)),
abducedFact(relD(john,v1,v1))
\end{lstlisting}
as a solution. This would lead to the general solution \begin{lstlisting}[frame = none]
abducedFact(relE(john,Y,Y)),
abducedFact(relD(john,Y,Y)),
abducedFact(relC(john,Y)
\end{lstlisting} which can still in fact be further generalised.\\ It is not hard however to come up with an example where this procedure will not in fact necessarily derive a most general solution. Consider the rule set given by the two rules:
\begin{lstlisting}[frame = none]
relA(X):-relB(X),relC(Y)
relA(X):-relB(X),relC(X).
\end{lstlisting}
Let $q$ be $relA(john)$, let there be no user provided facts and let $C$ be empty and suppose $N = 4$. Further say that as before, no instance of the predicate $relA$ may be abduced. Then given the ‘semi-res’ encoding, the solver would produce two optimal solutions:
\begin{lstlisting}[frame = none]
abducedFact(relB(john))
abducedFact(relC(extVar))
\end{lstlisting} and 
\begin{lstlisting}[frame = none]
abducedFact(relB(john))
abducedFact(relC(john))
\end{lstlisting}

However only the first answer from the solver leads to the most general
solution for the abduction problem. In this particular instance the issue is
that the original rule set contains a redundancy. The second rule is clearly a
specific instance of the first rule thus making the second rule
superfluous. It would be interesting to investigate under what conditions, on
the input rule set and the other parameters, does the procedure outlined above
for generalising the abductive solution produced, actually give some solution
that cannot be generalised further. \\Note that some of these generalised
abductive solutions can be produced by using skolem functions in the
abducibles generation encoding but that necessarily requires some form of
depth control to avoid infinite answer sets for some rule sets. With this
method of using the ‘semi-res’ encoding we can derive these generalised
solutions while also letting go of depth control by for example, deleting the
integer parameter of the ‘query’ , ‘createSub’ and ‘explains’ predicates
since, we will still always get finiteness of answer sets. As mentioned
earlier, at the moment, these are rather informal ideas but we believe they
provide an interesting avenue for future formal investigations.

%%% Local Variables:
%%% mode: latex
%%% TeX-master: "main_extended"
%%% End:

\end{document}